\documentclass{article}

\usepackage{microtype}
\usepackage{graphicx}
\usepackage[skip=-3pt]{subcaption}
\usepackage{booktabs} 

\usepackage{hyperref}
\hypersetup{%
	colorlinks=true,%
	linkcolor=black,%
	urlcolor=black,%
	citecolor=black,%
	filecolor=black%
}					
\usepackage[all]{hypcap}		

\usepackage[algo2e,ruled]{algorithm2e}

\usepackage[accepted]{icml2024}

\usepackage{amsmath}
\usepackage{amssymb}
\usepackage{mathtools}
\usepackage{amsthm}

\usepackage[most]{tcolorbox}
\usepackage{stackengine}
\usepackage{bbm}

\usepackage{supertabular, booktabs}
\usepackage{pifont}
\newcommand{\cmark}{\ding{51}}%
\newcommand{\xmark}{\ding{55}}%

\usepackage[capitalize,noabbrev,nameinlink]{cleveref}

\makeatletter
\def\thm@space@setup{\thm@preskip=0.3cm
\thm@postskip=0.0cm}
\makeatother

\usepackage[createShortEnv]{proof-at-the-end}
\theoremstyle{plain}
\newtheorem{theorem}{Theorem}[section]
\newtheorem{proposition}[theorem]{Proposition}
\newtheorem{lemma}[theorem]{Lemma}

\newcounter{hyp}
\newtheorem{hypothesis}[hyp]{Hypothesis}

\newenvironment{proofsketch}{%
  \renewcommand{\proofname}{Proof sketch}\proof}{\endproof}

\theoremstyle{definition}
\newtheorem{definition}[theorem]{Definition}

\theoremstyle{remark}

\usepackage[disable,textsize=tiny]{todonotes}

\icmltitlerunning{Transport of Algebraic Structure to Latent Embeddings}

\usepackage{enumitem}
\usepackage{dsfont}
\usepackage{tikz}
\usetikzlibrary{arrows,shapes,intersections}
\usepackage{tkz-euclide}	
\usepackage{multirow}
\usepackage{scalerel}		

\usepackage{float}
\newfloat{algorithm}{t}{lop}

\newcommand\blfootnote[1]{%
  \begingroup
  \renewcommand\thefootnote{}\footnote{#1}%
  \addtocounter{footnote}{-1}%
  \endgroup
}

\newcommand\tablelinestretch{0.15cm}

\newcommand{\R}{{\mathbb R}}
\newcommand{\N}{{\mathbb N}}

\newcommand{\I}{\mathbbm{1}}

\DeclareMathOperator*{\E}{\mathbb{E}}

\newcommand{\vol}{{\protect\ooalign{\hfil$V$\hfil\cr\kern0.08em--\hfil\cr}}}

\newcommand{\pow}{\mathcal{P}}

\DeclareMathOperator{\loss}{Loss}
\DeclareMathOperator{\lossval}{\mathfrak{L}}

\newcommand{\Rl}{\mathbb{R}^l}	
\newcommand{\Rd}{\mathbb{R}^d}	
\DeclareMathOperator{\ar}{ar}	
\newcommand{\F}{\mathcal{F}}	
\renewcommand{\S}{\mathcal{S}}	
\newcommand{\M}{\mathcal{M}}	
\renewcommand{\L}{\mathcal{L}}	
\newcommand{\A}{\mathcal{A}}	
\newcommand{\B}{\mathcal{B}}	
\renewcommand{\P}{\mathcal{P}}	
\newcommand{\iso}{\varphi}	

\newcommand{\C}{\mathcal{C}}	
\newcommand{\Utrue}{U_{\textup{true}}}	
\newcommand{\Upred}{U_{\textup{pred}}}	
\newcommand{\zpred}{z_{\textup{pred}}}	
\DeclareMathOperator{\iou}{IoU}	
\DeclareMathOperator{\acc}{Acc}	

\newcommand{\FB}{\F_{\textup{Bool}}} 
\newcommand{\FD}{\F_{\textup{Dist}}} 

\newcommand\thickcdot[1][.75]{\mathbin{\ThisStyle{\vcenter{\hbox{%
  \scalebox{#1}{$\SavedStyle\bullet$}}}}}%
}	

\newcommand{\redcirc}{\tikz\draw[red,fill=red] (0,0) circle (.5ex);}
\newcommand{\bluecirc}{\tikz\draw[blue,fill=blue] (0,0) circle (.5ex);}
\newcommand{\greencirc}{\tikz\draw[green!50!black,fill=green!50!black] (0,0) circle (.5ex);}
\newcommand{\tightenborder}{\hspace*{-0.5em}} 
\newcommand{\encodercolor}{orange!30} 

\begin{document}

\twocolumn[
\icmltitle{Transport of Algebraic Structure to Latent Embeddings}



\icmlsetsymbol{equal}{*}

\begin{icmlauthorlist}
\icmlauthor{Samuel Pfrommer}{aff}
\icmlauthor{Brendon G.\ Anderson}{aff}
\icmlauthor{Somayeh Sojoudi}{aff}
\end{icmlauthorlist}

\icmlaffiliation{aff}{University of California, Berkeley}

\icmlcorrespondingauthor{Samuel Pfrommer}{\texttt{sam.pfrommer@berkeley.edu}}

\icmlkeywords{Universal Algebra, Transport of Structure, Latent Embeddings, Implicit Neural Representations}

\vskip 0.3in
]



\printAffiliationsAndNotice{}  

\begin{abstract}
Machine learning often aims to produce latent embeddings of inputs which lie in a larger, abstract mathematical space. For example, in the field of 3D modeling, subsets of Euclidean space can be embedded as vectors using implicit neural representations. Such subsets also have a natural algebraic structure including operations (e.g., union) and corresponding laws (e.g., associativity). How can we learn to ``union'' two sets using only their latent embeddings while respecting associativity? We propose a general procedure for parameterizing latent space operations that are provably consistent with the laws on the input space. This is achieved by learning a bijection from the latent space to a carefully designed \emph{mirrored algebra} which is constructed on Euclidean space in accordance with desired laws. We evaluate these \emph{structural transport nets} for a range of mirrored algebras against baselines that operate directly on the latent space. Our experiments provide strong evidence that respecting the underlying algebraic structure of the input space is key for learning accurate and self-consistent operations. 
\end{abstract}

\section{Introduction}
\label{sec: intro}
Algebraic structure underpins a wide range of interesting mathematical objects such as sets, functions, distributions, and symbolic strings. In machine learning (ML), these objects are often learned and subsequently embedded into Euclidean space for downstream tasks: consider embeddings of implicit neural representations (INRs) for sets \citep{de2023deep}, hypernetworks for functions \citep{ha2016hypernetworks}, conditional embeddings of generative architectures for probability distributions \citep{sohn2015learning,winkler2019learning,nichol2021glide}, and text embeddings for strings \citep{wang2022text,devlin2018bert}. Our goal is to enable mathematical operations from the underlying algebraic structure (e.g., set union when the underlying objects are sets) to be applied directly to latent embeddings in a way that respects axiomatic laws.

The importance of respecting mathematical structure has motivated machine learning developments of immense importance. Indeed, much of geometric deep learning is directly driven by symmetries in underlying objects \citep{bronstein2021geometric}. Graph neural networks learn functions that provably respect equivariance or invariance properties under node-relabeling graph isomorphisms \citep{maron2018invariant,azizian2020expressive}. The seminal DeepSet architecture enforces permutation invariance, reflecting the unordered nature of its finite set inputs \citep{zaheer2017deep}. Convolutional filters are also known to be approximately equivariant to translations in input images---a structure which naturally mirrors that of the underlying image manifold \citep{cohen2016group,cohen2019general,kondor2018generalization}.

This work is a first attempt to transport general algebraic structures from input data onto learned latent embeddings. We outline a general procedure for defining algebraic \emph{operations} on the latent space that respect \emph{laws} on the \emph{source space} (input space). Defining operations directly on latent space embeddings, rather than using the original source objects, is crucial for computational efficiency and compatibility with larger ML workflows. There has been some interest in algebraic and category theoretic approaches to the study of specific computational architectures and automatic differentiation \citep{martin2018algebraic,shiebler2021category,sennesh2023computing}, as well as in the application of ML to computational problems arising in algebra \citep{he2023learning}. However, to the best of our knowledge, our work provides the first general method to transport algebraic structures to learned embeddings.

We discuss our ideas using the language of \emph{universal algebra}, which studies algebraic structures as general pairings of a set with a collection of operations \citep{burris1981course}. We note that universal algebra is subsumed within category theory. As the universal algebraic perspective is sufficient here, we avoid generalizing to more complex category-theoretic frameworks.\blfootnote{Source code for our experiments is available \href{https://github.com/spfrommer/latent_algebras/tree/main}{on GitHub}.}

As our transport of algebraic structures relies on the construction of a bijection map, we leverage architectures from the invertible neural network literature. Our model of choice is the seminal NICE architecture, which uses coupling layers to enable easily-computable forward and inverse methods \citep{dinh2014nice}. These coupling layers have been shown to be universal diffeomorphism approximators \citep{teshima2020coupling}, and are best known for their usefulness in constructing normalizing flows \citep{papamakarios2021normalizing,kobyzev2020normalizing}. Since our application requires differentiation through the function inverse, other architectures which rely on solving fixed-point iterations to compute inverses are not considered \citep{behrmann2019invertible}.

We focus on embeddings of positive-volume subsets of $\Rd$ as a working example.
This is distinct from methods that consider finite sets, such as DeepSets \citep{zaheer2017deep}. 
Our setting is motivated by the practical application of learning shapes for 3D modeling and graphics \citep{park2019deepsdf}. Typical approaches parameterize a signed distance function or simply regress on a shape indicator function \citep{park2019deepsdf,mescheder2019occupancy,chen2019learning}. As the object surface is implicitly defined as a level set of the resulting network, this is termed an \emph{Implicit Neural Representation} (INR). A subsequent innovation that we adopt improves representation quality by introducing sinusoidal activations \citep{sitzmann2020implicit}. While implicit representations of shapes achieve strong performance for a variety of objects, the significant storage requirements of the corresponding networks are impractical for larger workflows. Recent research has addressed this by directly compressing INR weights into latent embeddings \citep{de2023deep}, enabling a variety of downstream tasks such as shape generation.

\subsection{Contributions}
Our work establishes the following contributions.
\begin{enumerate}
    \item We develop a general procedure for transporting algebraic structure from the source data to the latent embedding space. This is accomplished via a learned bijection to a carefully designed mirrored algebra.
    \item We illustrate the subtleties that arise with this procedure by considering algebras of sets as a case study. Namely, we mathematically prove that transporting all three basic set operations (union, intersection, and complementation) is infeasible and subsequently drop complementation, yielding a distributive lattice structure on the source space which is transportable.
    \item We experimentally validate \cref{hyp: lawsat} on this distributive lattice of sets, showing that adherence to source algebra laws is crucial for strong learned operation performance.
\end{enumerate}

\begin{tcolorbox}[width=\linewidth,enhanced,interior style={color={black!2}}]
\begin{hypothesis} \label{hyp: lawsat}
    Learned latent space operations will achieve higher performance if they are constructed to satisfy the laws of the underlying source algebra.
\end{hypothesis}
\end{tcolorbox}

\section{Universal algebra primer}

In this section, we briefly recall the pertinent definitions and notations used throughout this paper. We refer the reader to \citet{burris1981course} and \citet{wechler2012universal} for detailed texts concerning universal algebra.

\paragraph{Algebras and isomorphisms.}
Let $A$ be a nonempty set and $n$ a nonnegative integer. If $n=0$, we define $A^n = \{\emptyset\}$. A function $f \colon A^n \to A$ is called an \emph{$n$-ary operation on $A$}, and $n$ is called the \emph{arity of $f$}. If the arity of $f$ is $1$, then $f$ is called a \emph{unary operation}, and if the arity of $f$ is $2$, then $f$ is called a \emph{binary operation}. If the arity of $f$ is $0$, then $f$ is called a \emph{nullary operation}, which may be identified with an element of $A$. We will commonly denote nullary operations, unary operations, and binary operations by $f = f(\emptyset)$, $f a = f(a)$, and $a f b = f(a,b)$, respectively.

A \emph{type} is a set $\F$, whose elements are called \emph{operation symbols}, together with a function $\ar \colon \F \to \N\cup\{0\}$. If $f\in \F$ and $\ar(f) = n$, then an $n$-ary operation $f^{\A} \colon A^n \to A$ is called a \emph{realization of $f$ on $A$}.

An \emph{algebra of type $\F$} is an ordered pair $\A = (A,\F^{\A})$ with $A$ being a nonempty set and $\F^{\A} = \{f^{\A} : f\in \F\}$ being a family of realizations $f^{\A}$ of operation symbols $f$ on $A$, and with $\F^{\A}$ in one-to-one correspondence with $\F$.

One of the most fundamental algebras is a \emph{group}, which is an algebra $(A,\thickcdot,{}^{-1},e)$ whose operations satisfy
\begin{align}
	e \thickcdot a &= a, \label{eq: G1} \tag{G1} \\
	(a^{-1}) \thickcdot a &= e, \label{eq: G2} \tag{G2} \\
	(a\thickcdot b) \thickcdot c &= a\thickcdot (b\thickcdot c), \label{eq: G3} \tag{G3}
\end{align}
for all $a,b,c\in A$. Here, $\thickcdot$ is a binary operation, ${}^{-1}$ is a unary operation, and $e$ is a nullary operation. The equations \eqref{eq: G1}, \eqref{eq: G2}, and \eqref{eq: G3} are the group's underlying \emph{laws}, which we will define shortly. We use the term \emph{algebraic structure} to refer to a combination of a type and a collection of laws.


Consider two algebras $\A = (A,\F^{\A})$ and $\B = (B,\F^{\B})$ of type $\F$. A function $\iso \colon A \to B$ is called an \emph{homomorphism from $\A$ to $\B$} if it satisfies
\begin{equation*}
	\iso(f^{\A}(a_1,\dots,a_n)) = f^{\B}(\iso(a_1),\dots,\iso(a_n))
\end{equation*}
for all $f\in\F$ and all $a_1,\dots,a_n\in A$, where of course $n=\ar(f)$. If, additionally, $\iso$ is bijective, then it is called an \emph{isomorphism from $\A$ to $\B$}. If $\A$ is isomorphic to $\B$ (meaning there is an isomorphism $\iso$ from $\A$ to $\B$), then we write $\A\cong \B$. Isomorphic algebras satisfy the same laws, and hence can be viewed as the same algebraic structures.

Two algebras may be of the same type yet not be isomorphic, and thus have fundamentally different structures. For example, rings and lattices are distinct algebraic structures of common type $\F = \{f_1,f_2\}$ with $\ar(f_1)=\ar(f_2)=2$.

\paragraph{Terms and laws.}
For a set of \emph{variables} $X$ and a type $\F$, the set $T_{\F}(X)$ is the set of \emph{terms of type $\F$ over $X$} and consists of all strings of variables in $X$ and nullary operations in $\F$, connected by $n$-ary operations. For example, consider a type $\F$ with one binary operation $\thickcdot$ and a nullary operation $e$. If $X = \{x, y\}$, then $x$, $y$, $e$, $x \thickcdot y$, $x \thickcdot (y \thickcdot e)$, and $x \thickcdot (x \thickcdot y)$ are all examples of terms in $T_{\F}(X)$.

Note that a term $p(x_1, \dots, x_n) \in T_{\F}(X)$ is defined independently of any specific algebra of type $\F$. Making the term concrete for a particular algebra $\A = (A,\F^{\A})$ of type $\F$ yields a \emph{term function} $p^{\A} \colon A^n \to A$. Namely, $p^{\A}(a_1, \dots, a_n)$ substitutes $a_i \in A$ for $x_i$ in the term $p(x_1, \dots, x_n)$, and recursively evaluates using the realized operations from $\A$. Continuing the previous example, let $\A$ be the group of the real numbers equipped with the standard addition operation. The term $p(x, y) = x \thickcdot (y \thickcdot e)$ would yield the term function given by $p^{\A}(a,b) = a + (b + 0)$.

We call two terms $p(x_1,\dots,x_n),q(x_1,\dots,x_n)\in T_{\F}(X)$ \emph{equivalent} with respect to an algebra $\A$ if, for all $a_i\in A$, it holds that $p^{\A}(a_1,\dots,a_n) = q^{\A}(a_1,\dots,a_n)$.

A \emph{law} $R$ for a type $\F$ is now defined as the equality of two terms $p(x_1,\dots,x_n), q(x_1,\dots,x_n) \in T_{\F}(X)$:
\begin{align*} \label{eq: lawdef}
    R: p(x_1, \dots, x_n) = q(x_1, \dots, x_n). 
\end{align*}
We use $R$ instead of the more common letter $L$, which we reserve for referring to latent spaces. For our running example, the commutative law for the underlying type $\F = \{\thickcdot,{}^{-1},e\}$ over a set of variables $X = \{x,y\}$ is given by
\begin{equation*}
	x\thickcdot y = y \thickcdot x.
\end{equation*}

Finally, we say that an algebra $\A$ of type $\F$ \emph{satisfies}, or \emph{respects}, a law $R : p(x_1,\dots,x_n) = q(x_1,\dots,x_n)$ if the law holds for realizations of the terms as term functions:
\begin{align*}
    R^{\A}: p^{\A}(a_1, \dots, a_n) = q^{\A}(a_1, \dots, a_n) ~ \text{for all $a_i\in A$}.
\end{align*}
It is clear that the group of reals under addition satisfies the commutative law, since $a + b = b + a$ for all $a,b\in \R$.





\section{Method} \label{sec: method}

With the framework of universal algebra now developed, we may formally describe the goal of this paper. Consider a machine learning task in which input data is drawn from a \emph{source algebra} $\S = (S,\F^{\S})$ of type $\F$. The canonical example we consider is that where input data takes the form of a set, and hence has associated operations of intersection, union, and complementation. The typical ML pipeline embeds source data from the source space $S$ into a Euclidean latent space $L = \Rl$. However, such latent space embeddings do not respect the algebraic structures encoded in $\S$; they are only endowed with the unrelated vector space structure of $\Rl$. Thus, the goal of this paper is as follows:

\begin{center}
\emph{Transport the algebraic structure $\S$ of the source space $S$ onto the latent space $L$.} 
\end{center}

Specifically, we seek to transport both the operations and laws of $\S$ onto $L$.
We emphasize that our goal of structural transport is distinct from constructing an isomorphism (or even a nontrivial homomorphism) $S \to L$; this is not generally possible, since $S$ is problem-determined and our setting assumes a pretrained encoder-decoder architecture which fixes $L$.

\SetKwInput{KwIn}{In}
\SetKwInput{KwOut}{Out}
\SetKwComment{Comment}{\# }{}
\begin{algorithm}
\caption{Transport of algebraic structure from $\S$ to $L$}
\label{alg: method}
\DontPrintSemicolon
\KwIn{Source alg.\ $\S$, latent space $L$, encoder $E$, decoder $D$}
\KwOut{Latent algebra $\L$}
\vspace*{0.2cm}

Fix mirrored space $M= \Rl$\; 

Select mirrored algebra $\M$\; \Comment*[r]{Same type as $\S$}

Parameterize bijection $\iso$\;

Define induced latent algebra $\L$\; \Comment*[r]{Via \eqref{eq: induced_operation}}

Learn parameters of $\iso$\; \Comment*[r]{Via \eqref{eq: learning}}
\end{algorithm}

\paragraph{Description of the method.} The general steps of our method are described in \cref{alg: method}, with a corresponding visualization in \cref{fig: method}. We assume that there is a fixed encoder $E \colon S \to L$ mapping source data to latent embeddings and a corresponding decoder $D$ (e.g., a pretrained autoencoder-style network).
To transport the algebraic structure from the source algebra $\S$ to the latent space $L$, we propose to learn a bijective map $\iso$ from $L$ to another space $M= \Rl$ of the same dimension. We may consider $M$ as an ``alternative latent space,'' albeit one in which we have complete design authority to impose operations that turn $M$ into an algebra $\M = (M,\F^{\M})$ of the same type $\F$ as $\S$.
Although we focus on the pretrained encoder-decoder setting for maximum flexibility, it is certainly possible to jointly learn $\iso$ together with the $E$ and $D$ in practice.

Concretely, we endow our \emph{mirrored space} $M$ with an $n$-ary operation $f^{\M}$ for each $n$-ary operation $f^{\S}$ from the source algebra. For an exemplar $\S$ with group structure, we would define one binary operation $\thickcdot^{\M}: \Rl \times \Rl \to \Rl$, one unary operation $({}^{-1})^{\M}: \Rl \to \Rl$, and one nullary operation identified with some element $e^{\M} \in \Rl$.


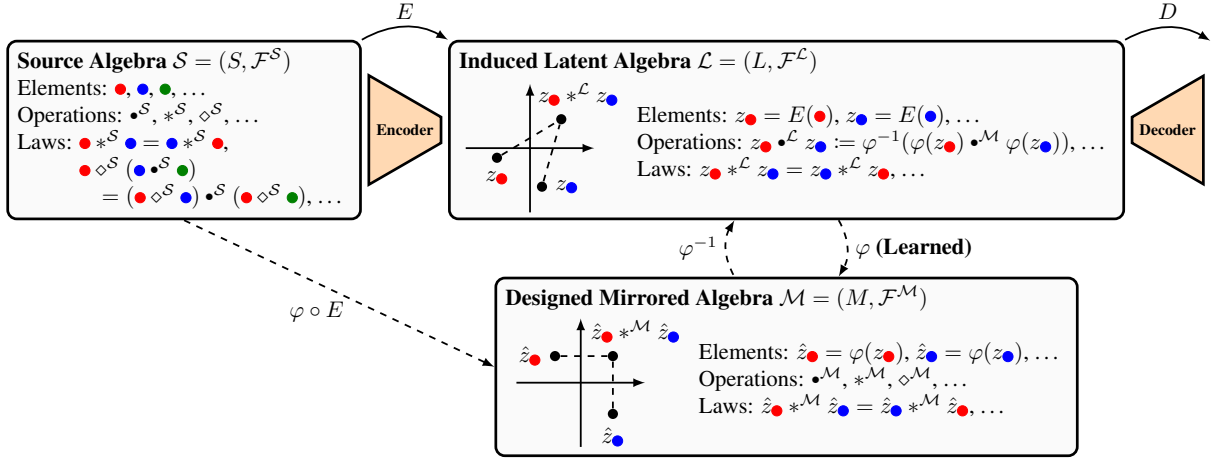
\begin{figure*}
    \begin{center}
	\resizebox{0.95\textwidth}{!}{
\begin{tikzpicture}
    [src/.style={very thick,black,fill=black!2,draw,rounded corners},
    enc/.style={trapezium,very thick,black,fill=\encodercolor,draw},
    lat/.style={very thick,black,fill=black!2,draw,rounded corners},
    mir/.style={very thick,black,fill=black!2,draw,rounded corners},
    dec/.style={trapezium,very thick,black,fill=\encodercolor,draw}]

    \node [src] (src) at (0,0) {
	    	\tightenborder
		\begin{tabular}{l}
		\textbf{Source Algebra $\S = (S,\F^{\S})$}\\
		Elements: \redcirc, \bluecirc, \greencirc, \dots\\
		Operations: $\thickcdot^{\S}$, $*^{\S}$, $\diamond^{\S}$, \dots\\
		Laws: $\redcirc *^{\S} \bluecirc = \bluecirc *^{\S} \redcirc$,\\
		\phantom{Laws: }$\redcirc \diamond^{\S} (\bluecirc \thickcdot^{\S} \greencirc)$\\
		\phantom{Laws: }\quad$= (\redcirc \diamond^{\S} \bluecirc) \thickcdot^{\S} (\redcirc \diamond^{\S} \greencirc)$, \dots
		\end{tabular}
	    	\tightenborder
		};

    \node[enc,
    shape border rotate = 270,
    trapezium stretches = true,
    minimum width = 1cm, 
    minimum height = 1cm,
    right=0.25em of src] (enc) {\scriptsize\textbf{Encoder}};

    \node [lat,right=0.25em of enc] (lat) {
	    	\tightenborder
		\begin{tabular}{l@{\hskip 0.25em}l}
		\multicolumn{2}{l}{\textbf{Induced Latent Algebra $\L = (L,\F^{\L})$}}\\
		\multirow{5}{*}{
			\begin{tikzpicture}
				\draw[-latex,thick] (-1,0) -- (1,0);
				\draw[-latex,thick] (0,-1) -- (0,1);
				\node[circle,fill=black,inner sep=0ex,minimum size=1ex,label=below:{$z_{\redcirc}$}] (x) at (-0.6,-0.15) {};
				\node[circle,fill=black,inner sep=0ex,minimum size=1ex,label=right:{$z_{\bluecirc}$}] (y) at (0.1,-0.6) {};
				\node[circle,fill=black,inner sep=0ex,minimum size=1ex,label=above:{\hspace*{1.5em}$z_{\redcirc} *^{\L} z_{\bluecirc}$}] (xy) at (0.4,0.45) {}; 
				\draw[thick,dashed] (x) -- (xy);
				\draw[thick,dashed] (y) -- (xy);
			\end{tikzpicture}
		}&\\
		&Elements: $z_{\redcirc} = E(\redcirc)$, $z_{\bluecirc} = E(\bluecirc)$, \dots\\
		&Operations: $z_{\redcirc} \thickcdot^{\L} z_{\bluecirc} \coloneqq \iso^{-1}(\iso(z_{\redcirc}) \thickcdot^{\M} \iso(z_{\bluecirc}))$, \dots\\
		&Laws: $z_{\redcirc} *^{\L} z_{\bluecirc} = z_{\bluecirc} *^{\L} z_{\redcirc}$, \dots\\
		&
		\end{tabular}
	    	\tightenborder
		};

    \node [mir,below=2.5em of lat] (mir) {
	    	\tightenborder
		\begin{tabular}{l@{\hskip 0.25em}l}
		\multicolumn{2}{l}{\textbf{Designed Mirrored Algebra $\M = (M,\F^{\M})$}}\\
		\multirow{5}{*}{
			\begin{tikzpicture}
				\draw[-latex,thick] (-1,0) -- (1,0);
				\draw[-latex,thick] (0,-1) -- (0,1);
				\node[circle,fill=black,inner sep=0ex,minimum size=1ex,label=left:{$\hat{z}_{\redcirc}$}] (x) at (-0.4,0.5) {};
				\node[circle,fill=black,inner sep=0ex,minimum size=1ex,label=below:{$\hat{z}_{\bluecirc}$}] (y) at (0.5,-0.4) {};
				\node[circle,fill=black,inner sep=0ex,minimum size=1ex,label=above:{\hspace*{2em}$\hat{z}_{\redcirc} *^{\M} \hat{z}_{\bluecirc}$}] (xy) at (0.5,0.5) {}; 
				\draw[thick,dashed] (x) -- (xy);
				\draw[thick,dashed] (y) -- (xy);
			\end{tikzpicture}
		}&\\
		&Elements: $\hat{z}_{\redcirc}=\iso(z_{\redcirc})$, $\hat{z}_{\bluecirc}=\iso(z_{\bluecirc})$, \dots\\
		&Operations: $\thickcdot^{\M}$, $*^{\M}$, $\diamond^{\M}$, \dots\\
		&Laws: $\hat{z}_{\redcirc} *^{\M} \hat{z}_{\bluecirc} = \hat{z}_{\bluecirc} *^{\M} \hat{z}_{\redcirc}$, \dots\\
		&
		\end{tabular}
	    	\tightenborder
		};

    \node[dec,
    shape border rotate = 90,
    trapezium stretches = true,
    minimum width = 1cm, 
    minimum height = 1cm,
    right=0.25em of lat] (dec) {\scriptsize\textbf{Decoder}};
    \node [right = 0.25em of dec] (phantom) {\vphantom{
		\begin{tabular}{l}
			1 \\ 2 \\ 3 \\ 4 \\ 5 \\ 6
		\end{tabular}
}};

    \path[latex-,draw,dashed,thick] (lat) to [bend right] node [midway,left] {$\iso^{-1}$} (mir);
    \path[-latex,draw,dashed,thick] (lat) to [bend left] node [midway,right] {$\iso$ \textbf{(Learned)}} (mir);
    \path[-latex,draw,dashed,thick] (src.south) to node [midway,below,xshift=-1em] {$\iso \circ E$} (mir.west);
    \path[-latex,draw,thick] (src.north east) to [bend left] node [midway,above] {$E$} (lat.north west);
    \path[-latex,draw,thick] (lat.north east) to [bend left] node [midway,above] {$D$} (phantom.north west);
\end{tikzpicture} }
    \end{center}
    \caption{The proposed method for transporting algebraic structure from $\S$ onto the latent space $L$. The bijection $\varphi$ is learned (hence the dashed arrows) in such a way as to best ``align'' the latent structure $\L$, induced from $\M$, with the given source structure $\S$. All other components are either fixed (e.g., the encoder and decoder) or designed \textit{a priori} (e.g., the mirrored algebra).}
    \label{fig: method}
	\vspace*{0.2cm}
\end{figure*}

We refer to the constructed $\M$ as the \emph{mirrored algebra}. Although it is always possible to endow $M$ with an algebra of the same type as $\S$, it is generally not possible to ensure that the resulting algebra $\M$ is isomorphic to $\S$. This may either be due to the fact that $S$ has cardinality strictly greater than $M$ (due to the embedding process $E$), or due to inherent incompatibilities between the laws of $\S$ and the natural Euclidean structure on $M$. Such incompatibilities are discussed in further detail with our case study in \cref{sec: set_case_study}. We note that the term ``mirrored algebra'' is our own and should not be conflated with other concepts in the literature.


We now transport the structure of our designed mirrored algebra $\M$ to the latent space $L$ via a learned bijection $\iso: L \to M$. Bijectivity is ensured by parameterizing $\iso$ as an invertible neural network using the architecture proposed in \citet{dinh2014nice}. This automatically induces an algebraic structure from $\M$ onto $L$. Namely, for every $n$-ary operation $f^{\M} \in \F^{\M}$, we define the realization $f^{\L} : L^n \to L$ of the corresponding operation symbol $f$ by
\begin{equation}
	f^{\L}(z_1,\dots,z_n) \coloneqq \iso^{-1}\big(f^{\M}(\iso(z_1),\dots,\iso(z_n))\big),
	\label{eq: induced_operation}
\end{equation}
for $z_1,\dots,z_n \in L$ and $\ar(f) = n$. Intuitively, the operation $f^{\L}$ is implemented by mapping latent embeddings into the mirrored space $M$, performing the corresponding operation $f^{\M}$ on these mirrored embeddings, and then pulling the result back to the latent space $L$. Of course, if $f^{\M}$ is a nullary operation $M$, then we define the corresponding operation $f^{\L}$ to be the nullary operation on $L$ given by $f^{\L}(\emptyset) = \iso^{-1}(f^{\M}(\emptyset))$.

\paragraph{Learning $\iso$.}
We briefly describe the process of learning $\iso$ to ``align'' the induced latent algebra $\L$ with the source algebra $\S$. Aligning $\L$ to $\S$ may be viewed as learning $\iso$ so that the laws of $\S$ are also satisfied by $\L$. To achieve this alignment, it suffices to align individual terms realized by $\S$ and $\L$, as laws are just equalities between terms. We propose the following procedure, which is illustrated in \cref{fig: learning}.

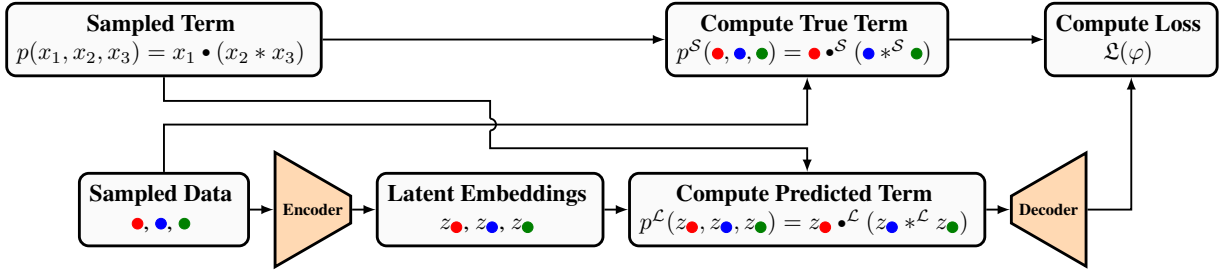
\begin{figure*}
    \begin{center}
	\resizebox{0.95\textwidth}{!}{
\def\radius{0.25em}
\begin{tikzpicture}
    [term/.style={very thick,black,fill=black!2,draw,rounded corners},
    data/.style={very thick,black,fill=black!2,draw,rounded corners},
    enc/.style={trapezium,very thick,black,fill=\encodercolor,draw},
    lat/.style={very thick,black,fill=black!2,draw,rounded corners},
    true/.style={very thick,black,fill=black!2,draw,rounded corners},
    pred/.style={very thick,black,fill=black!2,draw,rounded corners},
    loss/.style={very thick,black,fill=black!2,draw,rounded corners},
    dec/.style={trapezium,very thick,black,fill=\encodercolor,draw}]

    \node [term] (term) at (0,0) {
	    	\tightenborder
		\begin{tabular}{c}
		\textbf{Sampled Term}\\
		$p(x_1,x_2,x_3) = x_1 \thickcdot (x_2 * x_3)$
		\end{tabular}
	    	\tightenborder
		};

    \node [data,below=4em of term] (data) {
	    	\tightenborder
		\begin{tabular}{c}
		\textbf{Sampled Data}\\
		\redcirc, \bluecirc, \greencirc
		\end{tabular}
	    	\tightenborder
		};

    \node[enc,
    shape border rotate = 270,
    trapezium stretches = true,
    minimum width = 1cm, 
    minimum height = 1cm,
    right=1em of data] (enc) {\scriptsize\textbf{Encoder}};

    \node [lat,right=1em of enc] (lat) {
	    	\tightenborder
		\begin{tabular}{c}
		\textbf{Latent Embeddings}\\
		$z_{\redcirc}$, $z_{\bluecirc}$, $z_{\greencirc}$
		\end{tabular}
	    	\tightenborder
		};

    \node [pred,right=1em of lat] (pred) {
	    	\tightenborder
		\begin{tabular}{c}
		\textbf{Compute Predicted Term}\\
		$p^{\L}(z_{\redcirc},z_{\bluecirc},z_{\greencirc}) = z_{\redcirc} \thickcdot^{\L} (z_{\bluecirc} *^{\L} z_{\greencirc})$
		\end{tabular}
	    	\tightenborder
		};

    \node [true,above=4em of pred] (true) {
	    	\tightenborder
		\begin{tabular}{c}
		\textbf{Compute True Term}\\
		$p^{\S}(\redcirc,\bluecirc,\greencirc) = \redcirc \thickcdot^{\S} (\bluecirc *^{\S} \greencirc)$
		\end{tabular}
	    	\tightenborder
		};

    \node [loss,right=4em of true] (loss) {
	    	\tightenborder
		\begin{tabular}{c}
		\textbf{Compute Loss}\\
		$\lossval(\iso)$
		\end{tabular}
	    	\tightenborder
		};

    \node[dec,
    shape border rotate = 90,
    trapezium stretches = true,
    minimum width = 1cm, 
    minimum height = 1cm,
    right=1em of pred] (dec) {\scriptsize\textbf{Decoder}};

    \node [above=2em of lat,shape=coordinate] (phantom) {};
    \node [above=3em of lat,shape=coordinate] (phantom1) {};
    \node [above=1em of lat,shape=coordinate] (phantom2) {};
    \node [above=2.25em of lat,shape=coordinate] (phantom3) {};
    \node [above=1.75em of lat,shape=coordinate] (phantom4) {}; 

    \path[-latex,draw,thick] (term.east) to (true.west);
    \path[-latex,draw,thick] (true.east) to (loss.west);
    \path[-latex,draw,thick] (data.east) to (enc.west);
    \draw[-latex,draw,thick,name path=line 1] (data.north) |- (phantom) -| (true.south);
    \path[-latex,draw,thick] (enc.east) to (lat.west);
    \path[-latex,draw,thick] (lat.east) to (pred.west);
    \path[-latex,draw,thick] (pred.east) to (dec.west);
    \draw[-latex,draw,thick] (dec.east) -| (loss.south);

    \draw[draw,thick] (term.south) |- (phantom1) -- (phantom3);
    \draw[-latex,draw,thick] (phantom4) -- (phantom2) -| (pred.north);
    \path[name path=line 2] (phantom3) -- (phantom4);
    \path[name intersections={of = line 1 and line 2}];
    \coordinate (S)  at (intersection-1);
    \path[name path=circle] (S) circle(\radius);
    \path [name intersections={of = circle and line 2}];
    \coordinate (I1)  at (intersection-1);
    \coordinate (I2)  at (intersection-2);
    \tkzDrawArc[color=black,thick](S,I2)(I1);
\end{tikzpicture} }
    \end{center}
    \caption{The bijection $\iso$ is learned to align true sampled terms $p_i^{\S}(s_1,\dots,s_{n_i})$ with predicted terms $D(p_i^{\L}(E(s_1),\dots,E(s_{n_i})))$.}
	\vspace*{0.4cm}
    \label{fig: learning}
\end{figure*}

Let $p_i(x_1,\dots,x_{n_i})\in T_{\F}(X_i)$ be a ``sampled'' term of type $\F$ over a variable set $X_i$. The manner in which this term is sampled is task-dependent, but it suffices to identify this term as a random string involving operation symbols from $\F$ and variables from $X_i$---see \cref{sec: experiments} for concrete examples. Next, consider data $s_1,\dots,s_{n_i} \in S$ sampled from the source space. The term is first realized on this source data by computing $p_i^{\S}(s_1,\dots,s_{n_i})$. The term is then also realized by the induced latent algebra as $D(p_i^{\L}(z_1,\dots,z_{n_i}))$, with $z_j = E(s_j)$. The loss between this prediction and the ground truth, as a function of the bijection $\iso$, is given by
\begin{equation*}
	\lossval_i (\iso) \coloneqq \loss\big(D(p_i^{\L}(z_1,\dots,z_{n_i})), p_i^{\S}(s_1,\dots,s_{n_i})\big),
\end{equation*}
for some appropriately chosen loss function $\loss$.

For example, if $S$ is the power set of $\Rd$ equipped with intersection and union, the true sampled term might be realized as $p_i^{\S}(s_1,s_2,s_3) = s_1\cap^{\S}(s_2\cup^{\S} s_3)$ for some subset data $s_1,s_2,s_3\subseteq\Rd$, where $\cap^{\S}$ and $\cup^{\S}$ are actual set intersection and union operations, and the corresponding predicted term would be given by $D(E(s_1)\cap^{\L}(E(s_2) \cup^{\L} E(s_3)))$, where $\cap^{\L}$ and $\cup^{\L}$ are the intersection and union realized in Euclidean space by efficient arithmetic operations.

The final learning problem then amounts to solving
\begin{equation}
	\inf_{\iso \in \Phi} \frac{1}{N} \sum_{i=1}^N \lossval_i(\iso),
	\label{eq: learning}
\end{equation}
for some parameterized class $\Phi$ of bijections.

\paragraph{Theoretical developments.}

\textEnd{
\subsection{Proofs for \cref{sec: method}} \label{app: method}
}

Our method comes equipped with theoretical guarantees that the induced latent algebra respects the underlying source algebra. First, we show that the induced algebra is \emph{always} isomorphic to the mirrored algebra by construction. Full proofs for all results are provided in \cref{app: proofs}.

\begin{propositionE}[][end,restate,text link=]
	\label{prop: isomorphism}
	Suppose that $L,M = \Rl$ and that $\iso\colon L\to M$ is a bijection. Let $\M=(M,\F^{\M})$ be an algebra of type $\F$ and define the family $\F^{\L} \coloneqq \{f^{\L} : f\in \F\}$ of $n$-ary operations on $L$ by \eqref{eq: induced_operation}. Then, $\iso$ is an isomorphism from the induced algebra $\L = (L,\F^{\L})$ to $\M$.
\end{propositionE}

\begin{proofE}
	Let $f\in \F$, let $n = \ar(f)$, and consider the realization $f^{\M}$ on $M$ and the realization $f^{\L}$ on $L$ induced by \eqref{eq: induced_operation}. Let $z_1,\dots,z_n\in L$. We have that
	\begin{align*}
		\iso(f^{\L}(z_1,\dots,z_n)) &= \iso(\iso^{-1}(f^{\M}(\iso(z_1),\dots,\iso(z_n)))) \\
		&= f^{\M}(\iso(z_1),\dots,\iso(z_n))
	\end{align*}
	by construction of the operation $f^\L$. Hence, we see that $\iso$ is an isomorphism from the induced algebra $\L$ to the algebra $\M$.
\end{proofE}

As a consequence of \cref{prop: isomorphism}, a well-constructed mirrored space induces an algebra $\L$ such that laws on the source space are satisfied.

\begin{theoremE}[][end,restate,text link=] 
	\label{thm: source_latent}
	Consider a source algebra $\S = (S, \F^{\S})$ of type $\F$, and let $\M = (M,\F^{\M})$ be a mirrored space such that every law $R$ satisfied by $\S$ is also satisfied by $\M$. Then, the induced latent algebra $\L$, defined by \eqref{eq: induced_operation}, also satisfies every such law $R$, for any bijection $\iso \colon L \to M$.
\end{theoremE}
\begin{proofE}
	Let $\iso \colon L\to M$ be a bijection, and let $\L$ be the induced latent algebra defined by \eqref{eq: induced_operation}. Let $R$ be a law that is satisfied by $\S$ (and hence satisfied by $\M$), given by
	\begin{equation*}
		p(x_1,\dots,x_n) = q(x_1,\dots,x_n).
	\end{equation*}
	By \cref{prop: isomorphism}, $\iso$ is an isomorphism from $\L$ to $\M$. Let $f\in\F$ be an arbitrary operation symbol. Then, by the properties of isomorphisms, it must be that
	\begin{equation*}
		\iso(f^{\L}(z_1,\dots,z_n)) = f^{\M}(\iso(z_1),\dots,\iso(z_n))
	\end{equation*}
	for all $z_1,\dots,z_n\in L$. Thus, it holds that
	\begin{equation*}
		\iso(p^{\L}(z_1,\dots,z_n)) = p^{\M}(\iso(z_1),\dots,\iso(z_n))
	\end{equation*}
	for all $z_1,\dots,z_n\in L$, and similarly,
	\begin{equation*}
		\iso(q^{\L}(z_1,\dots,z_n)) = q^{\M}(\iso(z_1),\dots,\iso(z_n))
	\end{equation*}
	for all such $z_1,\dots,z_n$. Therefore, since $\M$ satisfies the law $R$, we conclude that
	\begin{equation*}
		\iso(p^{\L}(z_1,\dots,z_n)) = \iso(q^{\L}(z_1,\dots,z_n))
	\end{equation*}
	for all $z_1,\dots,z_n\in L$. Hence, by invertibility of $\iso$, we also find that
	\begin{equation*}
		p^{\L}(z_1,\dots,z_n) = q^{\L}(z_1,\dots,z_n)
	\end{equation*}
	for all $z_1,\dots,z_n\in L$, and therefore $\L$ satisfies the law $R$.
\end{proofE}
\begin{proofsketch}
	For a law 
	$p(x_1, \dots, x_n) = q(x_1, \dots, x_n)$
	which is satisfied by $\mathcal{M}$, we want to show that 
	$p^{\mathcal{L}}(z_1, \dots, z_n) = q^{\mathcal{L}}(z_1, \dots, z_n)$
	for all $z_i \in L$. Proposition 3.1 implies that 
	\[
		\varphi(p^{\mathcal{L}}(z_1, \dots, z_n)) = p^{\mathcal{M}}(\varphi(z_1), \dots, \varphi(z_n)).
	\]
	After applying a similar procedure to $q$, we can use the fact that $R$ is satisfied by $\mathcal{M}$ to conclude that 
	\[
		\varphi(p^{\mathcal{L}}(z_1, \dots, z_n)) = \varphi(q^{\mathcal{L}}(z_1, \dots, z_n)).
	\]
	Inverting by $\varphi$ concludes the proof.	
\end{proofsketch}

Unfortunately, there is no general guarantee that an isomorphism, or even a nontrivial homomorphism, exists from the source algebra $\S$ to the induced algebra on $L$, even when the mirrored algebra satisfies the same laws as $\S$.

\begin{propositionE}[][end,restate,text link=]
	\label{prop: nonhomomorphic}
	There exists a source algebra $\S = (S,\F^{\S})$ and a mirrored algebra $\M=(M,\F^{\M})$ with $M=\Rl$, both of the same type $\F$, such that $\M$ satisfies every law $R$ that $\S$ satisfies, and, for all bijections $\iso\colon L\to M$, there is no nontrivial homomorphism $\chi \colon S\to L$ when $L=\Rl$ is equipped with the algebra induced by $\M$ via \eqref{eq: induced_operation}.
\end{propositionE}

\begin{proofE}
	We prove the claim by construction. Consider the source algebra $\S = (\R,\thickcdot)$, with a sole binary operation $\thickcdot$ defined by
	\begin{equation*}
		s_1 \thickcdot s_2 = |s_1| s_2,
	\end{equation*}
	where, of course, $|s_1|$ represents the absolute value of the real number $s_1$, and $|s_1| s_2$ represents the usual product of the two real numbers $|s_1|$ and $s_2$. This source algebra is a semigroup, namely, $\R$ is closed under the binary operation $\thickcdot$, and $\thickcdot$ satisfies the associativity law, since
	\begin{align*}
		s_1\thickcdot (s_2\thickcdot s_3) &= |s_1|(s_2\thickcdot s_3) \\
						  &= |s_1| (|s_2|s_3) \\
						  &= |s_1 s_2|s_3 \\
						  &= \big| |s_1| s_2 \big| s_3 \\
						  &= (|s_1|s_2)\thickcdot s_3 \\
						  &= (s_1\thickcdot s_2)\thickcdot s_3
	\end{align*}
	for all $s_1,s_2,s_3\in\R$. Now, consider the mirrored algebra $\M = (\R,+)$, with $+$ being the standard addition operation on the real numbers. Obviously, $\M$ is also a semigroup, since $\R$ is closed under $+$, and $+$ is associative.

	We now show that $\M$ satisfies every law $R$ that $\S$ does. Let $R$ be a law for type $\F$ defined by
	\begin{equation*}
		R : p(x_1,\dots,x_n) = q(x_1,\dots,x_n)
	\end{equation*}
	for some arbitrary terms $p(x_1,\dots,x_n), q(x_1,\dots,x_n)\in T_{\F}(X)$. Suppose that $\S$ satisfies the law $R$. Then,
	\begin{equation*}
		p^{\S}(s_1,\dots,s_n) = q^{\S}(s_1,\dots,s_n)
	\end{equation*}
	for all $s_i\in \R$. Since the associative binary operation $\thickcdot$ is the only operation in $\F^{\S}$, it must be that the term function $p^{\S}$ is given by some repeated application of $\thickcdot$:
	\begin{equation*}
		p^{\S}(s_1,\dots,s_n) = s_{i_1}\thickcdot s_{i_2} \thickcdot \cdots \thickcdot s_{i_{m_p}}
	\end{equation*}
	for some $m_p\in\N$ and some tuple $(i_1,\dots,i_{m_p})\in\{1,\dots,n\}^{m_p}$. Similarly,
	\begin{equation*}
		q^{\S}(s_1,\dots,s_n) = s_{j_1}\thickcdot s_{j_2} \thickcdot \cdots \thickcdot s_{j_{m_q}}
	\end{equation*}
	for some $m_q\in \N$ and some tuple $(j_1,\dots,j_{m_p})\in\{1,\dots,n\}^{m_p}$. Thus,
	\begin{equation}
		|s_{i_1} \cdots s_{i_{m_p}-1} | s_{i_{m_p}} = |s_{j_1} \cdots s_{j_{m_q}-1} | s_{j_{m_q}}.
		\label{eq: semigroup_law}
	\end{equation}
	If $m_p > m_q$, then there exists some factor $s_i$ appearing in the product $|s_{i_1}\cdots s_{i_{m_p}-1}| s_{i_{m_p}}$ at least once more than it does in the product $|s_{j_1}\cdots s_{j_{m_q}-1}| s_{j_{m_q}}$. Thus, if the equality \eqref{eq: semigroup_law} holds for some $s_1,\dots,s_n\in S$, doubling this particular value $s_i$ would result in the law being violated, as the left-hand side of \eqref{eq: semigroup_law} would have an extra factor of $2$ that the right-hand side would not. This implies that $m_p \le m_q$. Analogous reasoning shows that $m_q \le m_p$, and hence it must be that $m_p = m_q$; the same number of factors appear in the left-hand and right-hand sides of the law's realizations. Furthermore, the same reasoning goes to show that the left-hand and right-hand products in \eqref{eq: semigroup_law} actually must contain exactly the same factors with the same multiplicity (albeit in possibly different order), i.e., the ordered tuple $(s_{i_1},\dots,s_{i_{m_p}})$ of real numbers is some permutation of the ordered tuple $(s_{j_1},\dots,s_{j_{m_q}})$. Hence, it must be the case that
	\begin{equation*}
		s_{i_1} + s_{i_2} + \cdots + s_{i_{m_p}} = s_{j_1} + s_{j_2} + \cdots + s_{j_{m_q}},
	\end{equation*}
	implying that
	\begin{equation*}
		p^{\M}(s_1,\dots,s_n) = p^{\M}(s_1,\dots,s_n).
	\end{equation*}
	That is, the law is satisfied by $\M$ as well. Since $R$ was arbitrarily chosen, we conclude that indeed $\M$ satisfies every law $R$ that $\S$ does.


	Now, let $\psi \colon S \to M$ be a homomorphism from $\S$ to $\M$. Then, it holds that
	\begin{equation*}
		\psi(s_1\thickcdot s_2) = \psi(s_1) + \psi(s_2)
	\end{equation*}
	for all $s_1,s_2 \in S = \R$. Therefore, for $s_1=0$ and $s_2 = s$ with $s\in S$ arbitrary, we conclude that
	\begin{equation*}
		\psi(0) = \psi(|0|s) = \psi(0 \thickcdot s) = \psi(0) + \psi(s),
	\end{equation*}
	and hence
	\begin{equation*}
		\psi(0) + \psi(s) = \psi(0) + \psi(t)
	\end{equation*}
	for all $s,t\in S$. However, since $+$ is the standard addition operation on $\R$, this is only possible if
	\begin{equation*}
		\psi(s) = \psi(t)
	\end{equation*}
	for all $s,t\in S$, meaning the homomorphism $\psi$ must be the trivial mapping $\psi \colon s\mapsto C$ with $C = \psi(0)$.

	Now, let $\iso \colon L\to M$ be an arbitrary bijection. Equip $L$ with the algebra $\L$ induced by $\M$, as defined by \eqref{eq: induced_operation}. Let $\chi \colon S \to L$ be a homomorphism from $\S$ to $\L$. Then, since $\iso$ is an isomorphism from $\L$ to $\M$ (per \cref{prop: isomorphism}), the composition $\iso\circ \chi$ is a homomorphism from $\S$ to $\M$. Therefore, by our analysis above, $\iso\circ \chi$ must be a trivial homomorphism given by $\iso\circ \chi \colon s \mapsto C$ with $C = \iso\circ \chi (0)$. Hence, for all $s\in S$, we conclude that
	\begin{equation*}
		\chi(s) = \iso^{-1}(C),
	\end{equation*}
	implying that $\chi$ must be a trivial homomorphism from $\S$ to $\L$. This concludes the proof.
\end{proofE}

On the other hand, under strong assumptions on the encoder and the expressibility of the source data within Euclidean space, we can guarantee the existence of a bijection $\iso$ that recovers an isomorphism $\S\cong \M \cong \L$, despite the fact that the encoder $E$ is fixed.

\begin{propositionE}[][end,restate,text link=]
	Consider a source algebra $\S = (S,\F^{\S})$ of type $\F$, the latent space $L=\Rl$, and an arbitrary encoder $E\colon S \to L$. If $E$ is bijective and there exists a mirrored algebra $\M = (M,\F^{\M})$ with $M = \Rl$ and an isomorphism $\psi \colon S \to M$, then there exists a bijection $\iso \colon L\to M$ such that $\iso \circ E$ equals the isomorphism $\psi$.
\end{propositionE}

\begin{proofE}
	Suppose that $E$ is bijective and that there exists a mirrored algebra $\M = (M,\F^{\M})$ with $M=\Rl$ and an isomorphism $\psi \colon S\to M$. Define $\iso \colon L\to M$ by $\iso(z) = \psi \circ E^{-1}(z)$, which is well-defined since $E$ is bijective. Then, it holds that
	\begin{equation*}
		\iso\circ E(s) = \psi(E^{-1}(E(s))) = \psi(s)
	\end{equation*}
	for all $s\in S$, which proves the result.
\end{proofE}

\paragraph{Limitations.}
There is a major challenge in transporting structure from $\S$ to $L$: the mirrored space structure may not be amenable to the structure that we want. We will demonstrate this in \cref{sec: set_case_study}, providing a general impossibility result as well as a specific corollary for the Boolean lattice setting. \cref{sec: experiments} experimentally explores this challenge and shows that even satisfying a subset of source algebra laws can still yield substantial benefits. At this point, it is also worth mentioning that our method requires the mirrored space to have the same dimension as the latent space, since our transport of structure depends on the invertibility of $\iso$. Generalizing past this restriction poses an interesting direction for future work.

\section{Case study: transporting algebras of sets}
\label{sec: set_case_study}

We apply our framework to learning the algebra of subsets of Euclidean space. This would empower neural networks to operate directly on \emph{subsets} of $\Rd$ \cite{de2023deep}. Conventional networks generally only operate \emph{pointwise}, producing a single output for a single input in $\Rd$. Allowing for sets to be tractably encoded and operated on unlocks new approaches for a variety of downstream tasks, such as prediction with set-valued uncertainties \cite{mahjourian2022occupancy}, reachable set computation \cite{meng2022learning}, safety-constrained trajectory optimization \cite{michaux2023reachability}, bin packing \cite{pan2023sdf}, object pile manipulation \cite{wang2023dynamic}, and swept volume approximation in robotics \cite{chiang2021fast}.

The purpose of our work is to illustrate the general principles behind structural transport nets and to experimentally test \cref{hyp: lawsat}. We thus do not specialize to any particular downstream application. Instead, this section explores the procedure for constructing a mirrored algebra via a concrete example, and \cref{sec: experiments} provides controlled synthetic experiments which support \cref{hyp: lawsat}.

\subsection{Lattices of sets}
We introduce here the algebraic structures that are considered in this section. A \emph{Boolean lattice} is an algebra $(A,\land,\lor,\lnot,0,1)$ such that the operations $\land$, $\lor$, and $\lnot$ satisfy the laws listed in \cref{tab: lattice_laws}. In a Boolean lattice, the binary operations $\land$ and $\lor$ are read ``meet'' and ``join,'' respectively, and the unary operation $\lnot$ is read ``not'' or ``complement.'' Since $0$ and $1$ are nullary operations, the ``$0$'' and ``$1$'' in the listed laws are to be interpreted as these operations' images $0(\emptyset)$ and $1(\emptyset)$ as elements in $A$. If $S$ is a set and $\P(S)$ is the power set of $S$, then $(\P(S),\cap,\cup,{}^c,\emptyset,S)$ is a Boolean lattice with ${}^c$ denoting set complementation. Dropping complementation and nullary operations yields a \emph{distributive lattice}, which is depicted in the upper section of \cref{tab: lattice_laws}.

We denote the Boolean lattice type as $\FB$, and the distributive lattice type as $\FD$.

\begin{table}
\caption{\label{tab: lattice_laws} Distributive and Boolean lattice laws.}
\medskip
\centering
\begin{tabular}{ll}
\toprule
\textbf{Commutativity} & $x \land y = y \land x$ \\
\textbf{Commutativity$^*$} & $x \lor y = y \lor x$ \\[\tablelinestretch] 
\textbf{Associativity} & $x \land (y \land z) = (x \land y) \land z$ \\
\textbf{Associativity$^*$} & $x \lor (y \lor z) = (x \lor y) \lor z$ \\[\tablelinestretch] 
\textbf{Absorption} & $x \lor (x \land y) = x$ \\
\textbf{Absorption$^*$} & $x \land (x \lor y) = x$ \\[\tablelinestretch] 
\textbf{Distributivity} & $x \lor (y \land z) = (x \lor y) \land (x \lor z)$ \\
\textbf{Distributivity$^*$} & $x \land (y \lor z) = (x \land y) \lor (x \land z)$ \\[\tablelinestretch] 
\midrule
\multicolumn{2}{c}{$\uparrow$ \quad Distributive lattice (without $0$, $1$, $\lnot$) \quad $\uparrow$} \\
\midrule
\textbf{Identity} & $x \land 1 = x$ \\
\textbf{Identity$^*$} & $x \lor 0 = x$ \\[\tablelinestretch] 
\textbf{Complementation} & $x \land (\lnot x) = 0$ \\
\textbf{Complementation$^*$} & $x \lor (\lnot x) = 1$ \\ 
\midrule
\multicolumn{2}{c}{$\uparrow$ \quad Boolean lattice (with $0$, $1$, $\lnot$) \quad $\uparrow$} \\
\bottomrule
\end{tabular}

\vspace*{0.0cm}
\end{table}

\subsection{Boolean lattice infeasibility} \label{sec: impossibility}
This section shows that it is impossible to define continuous operations on a Euclidean mirrored space $M = \Rl$ with the type $\FB$ such that the laws in \cref{tab: lattice_laws} are satisfied. Specifically, it is impossible to define a continuous involution with no fixed point, conflicting with complementation laws. We prove this using results from homology and provide both a general statement of the result and its specific implementations for Boolean lattices.

Restricting ourselves to continuous operations is important, as the complementation operation itself is continuous with respect to a natural topology on the space of sets; we explore this more thoroughly in \Cref{app: impossibility}. A more intuitive justification arises by noting that small perturbations to a set $A$ will yield commensurate perturbations to $A^c$.

Our first result shows, informally, that it is impossible to realize an algebra with a fixed point-free involution on the mirrored space using continuous operations. We refer the reader to \cref{app: impossibility} for more details and proofs.

\textEnd{
    \clearpage
\subsection{Proofs for \cref{sec: set_case_study}} \label{app: impossibility}
\paragraph{Continuity of complementation.}
We briefly explain why we only consider continuous operations on the mirrored space to represent the complementation operation $\lnot$ in the Boolean lattice type $\FB$. Consider the complementation operation ${}^c: \pow(\Rd) \to \pow(\Rd)$. As there is no ambient topology on $\pow(\Rd)$, the continuity of ${}^c$ is not well-defined in this context. However, a natural topology arises by passing first to the Borel sets $\Sigma$ on $\Rd$ and then quotienting by the null ideal $\Sigma \cap \mathcal{N}$, with set intersection, union, and complementation inherited in the natural way. Following \citet[323A.iii.e]{fremlin2000measure}, this allows us to define a topology via the \emph{symmetric difference metric}:
\begin{equation*}
    d(A, B) \coloneqq \mu(A \triangle B) ~ \text{for all $A, B \in \Sigma / \Sigma \cap \mathcal{N}$},
\end{equation*}
where $\mu$ is inherited naturally on the quotiented space from the Borel measure on $\Rd$. Since the symmetric difference between $A$ and $B$ satisfies $A \triangle B = (A^c) \triangle (B^c)$, it is immediate that the inherited complementation operation is continuous with respect to this topology.

\paragraph{Negative result.}
Before proving our main result, we first present a key result from the algebraic topology literature.
\begin{proposition} \label{prop: fixedpoint}
    Any continuous involution on $\R^n$ has a fixed point.
\end{proposition}
\begin{proof}[Proof of \cref{prop: fixedpoint}]
    This is an easy application of Theorem 9 in \citet{jaworowski1956antipodal}. Namely, $\R^n$ is a separable metric space, and it is acyclic because it is contractible.
\end{proof}

\begin{lemma} \label{lem: lnot_involution}
    Consider a Boolean lattice $\A = (A,\land,\lor,\lnot,0,1)$ of type $\FB$ and its associated laws in \cref{tab: lattice_laws}. Then, it holds that $\lnot$ is an involution with no fixed points.
\end{lemma}
\begin{proof}[Proof of \cref{lem: lnot_involution}]
    Clearly, the Boolean lattice laws in \cref{tab: lattice_laws} imply that $\lnot(\lnot a) = a$ for all $a\in A$, and thus $\lnot$ is an involution. Now assume for the sake of contradiction that $\lnot$ has a fixed point $b\in A$, so that $\lnot b = b$. Then, by the Boolean lattice laws,
    \begin{equation*}
        (\lnot b) \land b = 0 \implies b \land b = 0 \implies b = 0,
    \end{equation*}
    and, similarly,
    \begin{equation*}
        (\lnot b) \lor b = 1 \implies b \lor b = 1 \implies b = 1.
    \end{equation*}
    This is a contradiction.
\end{proof}

We are now ready to prove a general negative result.
}

\newcommand{\smallsquare}{{\scriptstyle \square}}

\begin{theoremE}[][end,restate,text link=] \label{thm: no_fixedpoint_involution}
    Consider an algebra $\A = (A, \F^{\A})$ with a unary operation $\smallsquare^{\A}$. Assume $\A$ satisfies laws $R_1, \dots, R_n$ which imply that $\smallsquare$ has no fixed point: $\smallsquare (x) \neq x$ for all $x \in A$. Furthermore, assume that one of the laws $R_i$ is the involution law given by
    \begin{equation*}
        \smallsquare(\smallsquare(x)) = x.
    \end{equation*}
    Then, there exists no algebra $\B = (B, \F^{\B})$ on the Euclidean space $B = \Rl$ such that $\smallsquare^{\B}$ is continuous and $R_1, \dots, R_n$ are all satisfied by $\B$.
\end{theoremE}
\begin{proofE}
	Suppose for the sake of contradiction that there exists an algebra $\B=(B,\F^{\B})$ on $B = \Rl$ such that $\smallsquare^{\B}$ is continuous and the laws $R_1,\dots,R_n$ are all satisfied by $\B$. Then, by assumption it must be the case that $\smallsquare^{\B}(b) \ne b$ for all $b\in B$. However, since $\smallsquare^{\B}$ is a continuous involution on $B = \Rl$, by \cref{prop: fixedpoint}, $\smallsquare^{\B}$ has a fixed point, i.e., $\smallsquare^{\B}(b) = b$ for some $b\in B$. This is a contradiction, and hence the result is proven. 
\end{proofE}

We provide a specific instantiation of the above theorem for our considered case of Boolean lattices, leveraging the fact that the complementation operation is unrealizable.

\begin{corollaryE}[][end,restate,text link=] \label{cor: no_euclidean_boollattice}
    The Boolean lattice type $\FB$ cannot be realized on $M = \Rl$ with continuous operations such that the Boolean lattice laws in \cref{tab: lattice_laws} are satisfied.
\end{corollaryE}
\begin{proofE}
    This follows directly from \cref{lem: lnot_involution} and \cref{thm: no_fixedpoint_involution}.
\end{proofE}
\begin{proofsketch}
    The Boolean not operator $\lnot$ satisfies $\lnot(\lnot a) = a$ for any $a$ in the domain. Furthermore, Boolean lattice laws show that if there were a fixed point $b = \lnot b$, we must have both $b=0$ and $b=1$; this is a contradiction and thus $\lnot$ is an involution with no fixed points. Applying \cref{thm: no_fixedpoint_involution} concludes the proof.
\end{proofsketch}

\subsection{Relaxing to a distributive lattice} \label{sec: relaxing}
\cref{sec: impossibility} shows that a Boolean lattice structure cannot be realized on $M = \Rl$. We relax our requirements to that of a distributive lattice, and present a structure known as a Riesz algebra that realizes $\FD$ and satisfies all associated laws.

\begin{definition} \label{def: repspace}
    The \emph{Riesz mirrored algebra} is the distributive lattice $\M = (M,\FD^\M)$ with operations given by
    \begin{align*}
        a \land^{\M} b = \min(a, b) \quad \text{and} \quad
        a \lor^{\M} b = \max(a, b)
    \end{align*}
    on $M = \Rl$, where $\min$ and $\max$ are defined elementwise. This algebra satisfies the distributive lattice laws in \cref{tab: lattice_laws}.
\end{definition}

Since our specific application concerns the distributed lattice of sets, we can equivalently take our operation symbols to be $\cap$ and $\cup$ in place of $\land$ and $\lor$, respectively. With this notation, the realization $\cap^{\S}: S \times S \to S$ is standard set intersection on $S = \pow(\Rd)$, the realization $\cap^{\M}: M \times M \to M$ is elementwise maximum on the mirrored space $M=\R^{l}$, and $\cap^{\L}: L \times L \to L$ is the operation on $L=\R^{l}$ induced via \eqref{eq: induced_operation}. Analogous notational identifications also hold for $\cup$.

\section{Experiments} \label{sec: experiments}

This section details our experimental results on transporting structure from algebras of sets to latent embeddings. Following the infeasibility result and subsequent structural relaxation in \cref{sec: set_case_study}, we seek to transport the distributive lattice defined by set intersection and set union, disregarding complementation. Our desired laws are listed in the upper section of \cref{tab: lattice_laws}, identifying $\land$ with $\cap$, and $\lor$ with $\cup$.

Our experiments explore the impact of different choices for mirrored algebra operations $\cap^{\M}$ and $\cup^{\M}$. \cref{sec: sat_laws} shows that operations that are well-aligned with source algebra laws outperform those that satisfy few laws, affirming \cref{hyp: lawsat}. \cref{sec: equivalent_expr} shows that well-designed mirrored algebras are crucial for ensuring \emph{self-consistency}: the property that equivalent terms produce the same prediction.

We now introduce the shared portions of the experimental setup, with further details deferred to \cref{app: experiments}. 

\paragraph{Candidate operations.} Our distributive lattice of sets contains two binary operations: meet ($\cap$) and join ($\cup$). We must realize these on the mirrored space as binary vector operations $\cap^{\M}$ and $\cup^{\M}$. We restrict ourselves to closed-form operations that are well-conditioned (as opposed to elementwise division or exponentiation, for example). The list of candidate operations in \cref{tab: candidates} includes the Riesz algebra $\min$ and $\max$ operations, as well as the standard vector operations of addition, subtraction, and Hadamard product. For diversity, we include an operation that is commutative but not associative (scaled addition), associative but not commutative (matrix product), and neither (cyclic addition).

\begin{table}
\caption{\label{tab: candidates} List of candidate operations on $M$.}
\medskip
\centering
\begin{tabular}{ll}
\toprule
\textbf{Element min ($\min$)} & $(a, b) \mapsto \min(a, b)$ \\
\textbf{Element max ($\max$)} & $(a, b) \mapsto \max(a, b)$ \\
\textbf{Addition ($+$)} & $(a, b) \mapsto a + b$ \\
\textbf{Subtraction ($-$)} & $(a, b) \mapsto a - b$ \\
\textbf{Hadamard prod. ($\odot$)} & $(a, b) \mapsto a \odot b$ \\
\textbf{Scaled addition ($+_s$)} & $(a, b) \mapsto 2 a + 2 b$ \\
\textbf{Matrix prod. ($\times_{\text{mat}}$)} & $(a, b) \mapsto \mathtt{sq}^{-1} (\mathtt{sq}(a) \cdot \mathtt{sq}(b))$ \\
\textbf{Cyclic addition ($+_c$)} & $(a, b) \mapsto \mathtt{roll}(a) + b$ \\
\bottomrule
\end{tabular}
\vspace*{0.3cm}
\end{table}

We define the function $\mathtt{sq} \colon \R^{l} \to \R^{\sqrt{l} \times \sqrt{l}}$ to reshape a vector into a square matrix (assuming $l$ is a square number), and $\mathtt{roll} \colon \Rl \to \Rl$ to cycle vector elements by one index. We denote the set of all candidate operations by
\begin{equation*}
    \C = \{ \min, \max, +, -, \odot, +_s, \times_{\text{mat}}, +_c \}.
\end{equation*}

\begin{figure*}
\centering
\begin{subfigure}[b]{0.44\linewidth}
    \centering
    \includegraphics[width=\linewidth]{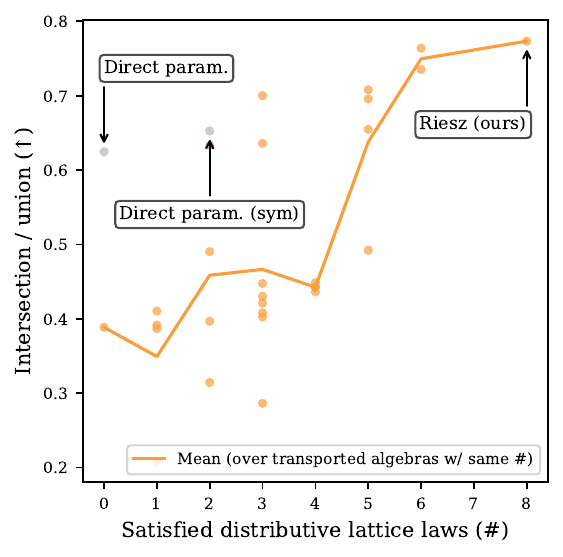}
    \caption{\label{fig: line_a}}
\end{subfigure}
\quad
\begin{subfigure}[b]{0.44\linewidth}
    \centering
    \includegraphics[width=\linewidth]{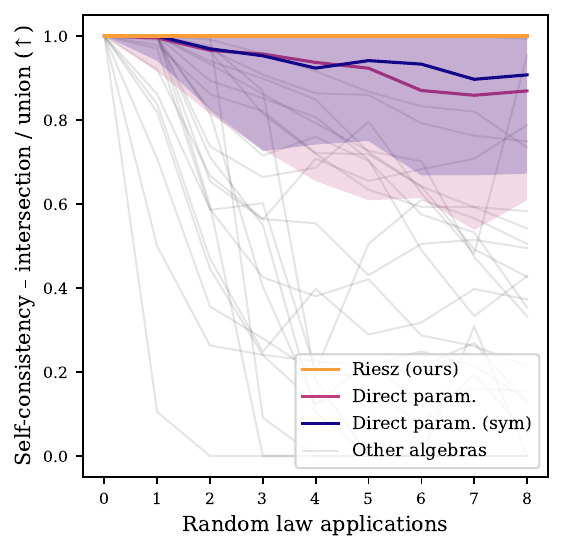}
    \caption{\label{fig: line_b}}
\end{subfigure}

\caption{\label{fig: line} \subref{fig: line_a}) Learned operation performance vs. satisfaction of distributive lattice laws (solid line is mean). \subref{fig: line_b}) Self-consistency vs. number of random symbolic manipulations (i.e., law applications). Solid lines are medians, shaded areas capture $20$th to $80$th percentile ranges.}
\vspace*{0.2cm}
\end{figure*}

\paragraph{Dataset.} To generate a synthetic random subset of $\Rd$ for $d=2$, we first uniformly sample two random integers $n_i, n_o$ from $\{1, 2, \dots, 10\}$. We then restrict ourselves to the zero-centered square and sample $n_i$ and $n_o$ points from $[-1, 1]^2$ to yield $I = \{v^i_1, \dots, v^i_{n_i}\}$ and $O = \{v^o_1, \dots, v^o_{n_o}\}$.
We then generate a set $U$ from these points as follows:
\begin{equation*}
    U = \Big\{ u \in [-1, 1]^2 : \min_{v \in I} \|v - u\|_2 \leq \min_{v \in O} \|v - u\|_2 \Big\}.
    \label{eq: randset}
\end{equation*}
We generate $10^4$ such random sets with an $80\%$ training, $10\%$ validation, and $10\%$ testing split. For each set, an INR is trained on evaluations of the set indicator function using a SIREN architecture \cite{sitzmann2020implicit}. An \texttt{inr2vec} architecture \citep{de2023deep} is then trained over this dataset, resulting in: 1) an encoder $E \colon S \to L$ mapping a set (as represented by the raw weight matrices of an INR) to a latent embedding space $L = \Rl$ with $l=1024$, and 2) a decoder $D \colon [-1,1]^2 \times L \to \R$ that predicts whether a particular point is in the set associated with a latent.

With some abuse of notation, we let $E(U) \in L$ denote the embedding of the INR trained on a set $U \subseteq [-1,1]^2$ as described above. We precompute and store latent embeddings for all INRs, after which the encoder is no longer required. Decoder weights are also fixed for our later experiments.

\paragraph{Parameterizations.} A particular training run starts with a fixed choice of operations $\cap^{\M}, \cup^{\M} \colon M \times M \to M$ on the mirrored space (e.g., $\cap^{\M} = \min$ and $\cup^{\M} = \max$ for the Riesz mirrored algebra). The learned bijection $\iso \colon L \to M$ is constructed as a modified NICE architecture \cite{dinh2014nice}. At training time, $\iso$ is the only learned component. Importantly, $\iso$ induces latent space operations $\cap^{\L}, \cup^{\L} \colon L \times L \to L$ from $\cap^{\M}, \cup^{\M}$ via \eqref{eq: induced_operation}.

For reference, we also try to \emph{directly parameterize} operations on the latent space as $\cap^{\L} = f_{\cap}$ and $\cup^{\L} = f_{\cup}$, with learned functions $f_{\cap}, f_{\cup}\colon L \times L \to L$. We compare two options for this parameterization. The first is simply constructing $f_{\cap}$ and $f_{\cup}$ as multilayer perceptrons on the vector concatenation of inputs (no law guarantees). The second involves parameterizing in a symmetric, commutativity-preserving manner via the form 
\[
    f(z_1,z_2) = h\big(g(z_1) + g(z_2)\big),
\]
where $h$ and $g$ are separate MLPs with compatible domains and codomains. We annotate this second parameterization using ``sym'' in our plots (see \citet{zaheer2017deep}).

\paragraph{Loss and metrics.} The training loss and evaluation metrics are computed over randomly constructed terms with a random number of starting symbols $\ell \in \{1, 2, \dots, {\ell}_{\max}\}$ (in our experiments, ${\ell}_{\max}=10$). We generate these by recursively combining random pairs of terms with either $\cap$ or $\cup$, starting with $\ell$ singleton terms (i.e., variables) and ending after $\ell-1$ operations when only the final combined term remains.

For a particular such term $p(x_1, \dots, x_{\ell})$, we fetch $\ell$ sets $U_1, \dots, U_{\ell}$ from data with corresponding precomputed \texttt{inr2vec} latent embeddings $z_1, \dots, z_{\ell}$, recalling that $z_i = E(U_i) \in L$. We evaluate the ground truth set $\Utrue \subseteq [-1, 1]^2$ via the realized term value
\[
    \Utrue = p^{\S}(U_1, \dots, U_{\ell}),
\]
taking $\cap^{\S}$ and $\cup^{\S}$ to be standard set-theoretic intersection and union. We similarly evaluate the predicted latent 
\begin{equation}
    \zpred=p^{\L}(z_1, \dots, z_{\ell}),
    \label{eq: z_pred}
\end{equation}
using $\cap^{\L}$ and $\cup^{\L}$, that are induced from $\cap^{\M}$ and $\cup^{\M}$ via \eqref{eq: induced_operation}. The predicted set is then given by
\begin{equation}
    \Upred = \{ u \in [-1, 1]^2 : D(u, \zpred) \geq 0 \}.
    \label{eq: U_pred}
\end{equation}
All metrics are then approximated using uniformly sampled $u \in [-1, 1]^2$. Our loss is the expectation of the binary cross-entropy loss against the ground truth set indicator function
\begin{equation*}
    \loss(\Upred, \Utrue) = \E \nolimits_{u} \left[ \text{BCE} \big( D(u, \zpred), \I_{\Utrue}(u) \big) \right],
\end{equation*}
and our intersection over union (IoU) metric is written as 
\begin{equation*}
    \iou(\Upred, \Utrue) = 
    \frac{
        \E \nolimits_{u}  [\I_{\Utrue \cap \Upred}(u)]
    }{
        \E \nolimits_{u}  [\I_{\Utrue \cup \Upred}(u)]
    }.
\end{equation*}
The IoU score ranges from zero to one (perfect prediction).

\subsection{Operation performance vs.\ structure choice} \label{sec: sat_laws}

\begin{table*}[ht]
    \centering
    \caption{\label{tab: algebra_short} Selection of candidate operations on the mirrored space with the satisfied distributive lattice laws (fully reproduced in \cref{tab: algebra_long}). Due to distributive lattice symmetries, we have two laws for each column (e.g., $a \cap^{\M} b = b \cap^{\M} a$ and $a \cup^{\M} b = b \cup^{\M} a$). The first row imposes a Riesz algebra structure. The second column counts how many laws are satisfied by a particular pair of operations.
    }
    \medskip
    \begin{tabular}{llrcccc}
\toprule
 Operations                        &                                   &   \# & Commutativity ($^*$)   & Associativity ($^*$)   & Absorption ($^*$)   & Distributivity ($^*$)   \\
\midrule
 $\cap^{\M} = \max$                & $\cup^{\M} = \min$                &    8 & \cmark \cmark          & \cmark \cmark          & \cmark \cmark       & \cmark \cmark           \\
 $\cap^{\M} = \max$                & $\cup^{\M} = \odot$               &    6 & \cmark \cmark          & \cmark \cmark          & \xmark \cmark       & \cmark \xmark           \\
 $\cap^{\M} = \min$                & $\cup^{\M} = +$                   &    6 & \cmark \cmark          & \cmark \cmark          & \xmark \cmark       & \cmark \xmark           \\
 $\cap^{\M} = \max$                & $\cup^{\M} = +$                   &    5 & \cmark \cmark          & \cmark \cmark          & \xmark \xmark       & \cmark \xmark           \\
 $\cap^{\M} = \min$                & $\cup^{\M} = \odot$               &    5 & \cmark \cmark          & \cmark \cmark          & \xmark \xmark       & \cmark \xmark           \\
 $\cap^{\M} = \min$                & $\cup^{\M} = +_s$                 &    5 & \cmark \cmark          & \cmark \xmark          & \xmark \cmark       & \cmark \xmark           \\
 $\cap^{\M} = +$                   & $\cup^{\M} = \odot$               &    5 & \cmark \cmark          & \cmark \cmark          & \xmark \xmark       & \cmark \xmark           \\
\multicolumn{7}{c}{\makebox[1em][c]{\vdots}} \\
 $\cap^{\M} = \times_{\text{mat}}$ & $\cup^{\M} = +_c$                 &    1 & \xmark \xmark          & \cmark \xmark          & \xmark \xmark       & \xmark \xmark           \\
 $\cap^{\M} = -$                   & $\cup^{\M} = +_c$                 &    0 & \xmark \xmark          & \xmark \xmark          & \xmark \xmark       & \xmark \xmark           \\
\bottomrule
\end{tabular}

    \vspace*{0.5cm}
\end{table*}

This experiment tests various candidate realizations of $\cap$ and $\cup$ on $M$, with the aim of evaluating whether satisfying distributed lattice laws induces superior performance. We consider all possible assignments $(\cap^{\M}, \cup^{\M}) \in \C\times \C$ with $\cap^{\M} \neq \cup^{\M}$, excluding flipped assignments (e.g., $(\max, \min)$ versus $(\min, \max)$) due to the exact symmetry of distributive lattice laws and our data generating process. This results in $\binom{|\C|}{2} = \binom{8}{2} = 28$ possible combinations. For each assignment $(\cap^{\M}, \cup^{\M})$, we determine which distributive lattice laws from \cref{tab: lattice_laws} are satisfied using numerical testing. We provide some illustrative examples in \cref{tab: algebra_short}, with the full list provided by \cref{tab: algebra_long} in the appendix.

Our results are depicted in \cref{fig: line_a}. Each dot represents a particular choice of operations (i.e., a particular mirrored algebra). The $x$-axis groups together algebras which satisfy the same number of distributive lattice laws (\# column in \cref{tab: algebra_long}). The $y$-axis reports the mean IoU performance of a particular algebra, averaged over random terms.

\cref{fig: line_a} provides clear experimental support for \cref{hyp: lawsat}: the accuracy of learned set operations is strongly tied to the number of satisfied source algebra laws. The Riesz algebra completely satisfies all $8$ laws and achieves the best performance, while operations with few satisfied laws struggle. Despite significantly underperforming the Riesz algebra, the direct latent parameterizations surpass transported algebras with a similar number of satisfied laws, suggesting that the flexibility of their parameterization somewhat mitigates their lack of algebraic structure. Interestingly, algebras that only violate a few laws substantially outperform algebras that violate most or all; there is a notable increasing trend in performance. Thus even when not all laws can be satisfied, a reasonably well-aligned mirrored algebra can still provide substantial benefits.

\subsection{Consistency under equivalent terms} \label{sec: equivalent_expr}
This experiment adopts the same setting as above, but considers a different question: how \emph{self-consistent} are the predictions of a model for terms that are distinct but equivalent with respect to $\FD$? Naturally, we expect a good model to provide the same predicted set for $A \cap B$ and $B \cap A$.

Consider a random term $p(x_1, \dots, x_{\ell})$, with sampled latents $z_1, \dots, z_{\ell}$ yielding a predicted set $\Upred$ via \eqref{eq: z_pred} and \eqref{eq: U_pred}. Instead of comparing $\Upred$ to $\Utrue$, we generate a family of equivalent terms $q_i(x_1, \dots, x_{\ell})$ by randomly selecting laws and substituting their expressions into $p(x_1,\dots,x_{\ell})$ if such expressions are present in $p(x_1,\dots,x_{\ell})$. For each equivalent term, we compare the new predicted set $V_{\textup{pred}}$ (computed via \eqref{eq: z_pred} and \eqref{eq: U_pred} as before) with the original prediction $\Upred$ and compute the corresponding IoU metric.

\cref{fig: line_b} summarizes our results. The $x$-axis represents the number of law applications, and the $y$-axis represents the self-consistency IoU. The solid lines represent the median performance for each choice of candidate operations, with the shaded areas representing the direct parameterizations' $20$-to-$80$th percentile ranges.

Our Riesz mirrored algebra is perfectly self-consistent, experimentally validating \cref{prop: isomorphism}. While the median performance of the learned baselines degrades moderately as the terms diverge, the bottom quartile drops sharply with even just two random symbolic manipulations. Interestingly, the direct parameterizations have a higher self-consistency than most other algebras, despite satisfying only zero or two laws. This suggests that a flexible parameterization can learn the appropriate symmetries to some degree, although we note that the Riesz algebra is decidedly superior to both across all experiments.

\section{Conclusion}
Interesting mathematical objects generally carry additional algebraic structure, such as operations and laws. Machine learning methods often encode such objects (sets, functions, etc.) into latent embeddings for downstream tasks. This paper examines the possibility of learning latent space operations that provably satisfy the same structural laws as the source algebra of input data. We provide a general procedure for constructing \emph{structural transport nets} to carry out such transport of structure, and we illustrate the method in a concrete case study of the algebra of sets. Experiment results support our key hypothesis: stronger alignment between latent space operations and source algebra laws improves the performance of learned operations.
Exciting future research involves further developing the theory of realizable latent-space operations and exploring downstream applications of structural transport nets.

\section*{Impact Statement}
This paper presents work whose goal is to advance the field of machine learning. There are many potential societal consequences of our work, none which we feel must be specifically highlighted here.

\bibliography{main}
\bibliographystyle{icml2024}

\newpage
\appendix
\onecolumn
\section{Proofs}
\label{app: proofs}
\printProofs

\clearpage
\section{Experiments} \label{app: experiments}
This appendix describes our experimental setup and provides additional figures conveying our results.

Reproducing these experiments takes approximately $5$ GPU-days on an RTX A6000 GPU with an i7 core CPU. Our codebase was developed against PyTorch 2.1.2.

\subsection{Hyperparameters and metrics} \label{app: experiment_hyperparameters}

This section details the training hyperparameters and model architectures for reproducing our results. When training any model, we later reload model weights from the checkpoint with the best validation loss. All of our reported metrics are evaluated against a separate testing set.

\paragraph{Implicit neural representations.} We train each INR using binary cross-entropy loss over the provided shape indicator function (\cref{sec: experiments}) sampled at $5000$ points in $[-1, 1]^2$. Our INR function is a standard sinusoidal activation SIREN network, with $128$ hidden neurons, $3$ layers, and $\omega_0=30$. Training uses the Adam optimizer with a learning rate of $0.01$ and batch size of $10^3$ for $10$ epochs. These hyperparameters were chosen to enable fast convergence (on the order of seconds), as this training process is repeated $10^4$ times in total. All INRs were trained using the same weight initialization for easier downstream embedding \citep{de2023deep}.

\paragraph{Latent embeddings (\texttt{inr2vec}).} After producing the dataset of INRs, we train one latent embedder over the training split. We reuse the architecture of \citet{de2023deep} to train a model which takes in all MLP weights for the input INR for completeness (as opposed to just the hidden layer weights). Our latent embedding dimension is $1024$--all other architecture hyperparameters are as in \citet{de2023deep}. We use the Adam optimizer for $50$ epochs with a learning rate of $10^{-4}$, weight decay of $10^{-3}$, and batch size of $16$.

\paragraph{Structural transport nets.} When learning a structural transport net, we have a fixed decoder $D$ and fixed mirrored algebra $\M$. The only learned parameters are those defining the map $\iso: \Rl \to \Rl$, which must be a bijection in order to achieve transport of structure and the isomorphism $\L \cong \M$. We parameterize $\iso$ as a series of additive coupling layers to enable easy differentiation through both the forward and inverse maps. We reuse the seminal NICE architecture, reducing the number of additive coupling layers to $2$ and the number of nonlinear layers to $3$ for efficiency \citep{dinh2014nice}.
Since our application requires differentiation through the function inverse, other architectures which rely on solving fixed-point iterations to compute inverses are not considered \citep{behrmann2019invertible}.
Training runs for $10$ epochs using Adam with a learning rate of $10^{-3}$.

\paragraph{Direct latent space-parameterized baselines.} Latent space-parameterized baselines are instantiated as functions $f_{\cap}, f_{\cup} \colon \Rl \times \Rl \to \Rl$. As described in \cref{sec: experiments}, we compare two parameterizations:
\begin{enumerate}
    \item Let each $f$ be realized by a standard multilayer perceptron (MLP) with ReLU nonlinearities, operating on the vector concatenation of the inputs.
    \item Let each $f$ take the form
    \[
        f(z_1, z_2) = h(g(z_1) + g(z_2)),
    \]
    with $g \colon \Rl \to \R^{256}$ and $h \colon \R^{256} \to \Rl$ MLPs.
\end{enumerate}

Notably, the second option satisfies commutativity via commutativity of the operation $+$, and thus satisfies two of the distributive lattice laws. We perform a hyperparameter sweep for each parameterization over the learning rates $\{10^{-4},10^{-3},10^{-2}\}$, layer counts $\{2,3,4\}$, and hidden dimension $\{64,128,256,512\}$; the final optimal hyperparameters set the learning rate to $10^{-4}$, layer count to $2$, and hidden dimension to $256$.

\paragraph{Computed metrics.} We discuss a few details of our reported metrics. While the IoU metric is standard, it can occasionally be undefined if the union volume is zero; we exclude these datapoints from our computed statistics. We also tested an accuracy metric
\[
    \acc(\Upred, \Utrue) = \E \nolimits_{u} \left[ \I_{\Upred}(u) \stackrel{?}{=} \I_{\Utrue}(u) \right],
\]
where $\stackrel{?}{=}$ outputs $1$ for equality and $0$ otherwise. The results were qualitatively similar to that of the IoU metric so we omit them for simplicity.

\subsection{Procedures}
\paragraph{Random term generation.} \cref{sec: method} and \cref{sec: experiments} describe a procedure for generating random terms of a particular type. \cref{alg: rand_expr} provides more detailed pseudocode for the set distributive lattice type $\FD$. We also work through an example for $\ell=5$ starting symbols.
\begin{alignat*}{6}
    v && \qquad w && \qquad x && \qquad y && \qquad z&& \\
    v && \qquad w && \qquad x && \qquad (y && \lor \;\; z&&) \\
    v && \qquad (w && \land \;\; x&&) \; \qquad (y && \lor \;\; z&&) \\
    (v && \lor \;\; (w && \land \;\; x &&)) \qquad (y && \lor \;\; z&&) \\
    (v && \lor \;\; (w && \land \;\; x &&)) \; \land \;\; (y && \lor \;\; z&&) \\
\end{alignat*}
The final line is the finished random term.

\paragraph{Equivalent term generation.}
The experiment in \cref{sec: equivalent_expr} outlines a scheme for randomly generating a term $q(x_1, \dots, x_{\ell})$ which is equivalent to some original term $p(x_1, \dots, x_{\ell})$. At a high level, this works by successively substituting terms from distributive lattice laws. \cref{alg: equiv_expr} provides informal pseudocode for this procedure, and rests on the idea of applying laws to manipulate terms; a formal treatment of this is outside the scope of our paper, and we will explain our ideas informally.

Whether a law \emph{applies} to a particular expression and how a law is then \emph{applied} to that expression depends on which law is considered. We provide a concrete example for the distributive laws, with other laws following similarly.

Consider the distributive laws
\begin{align} \label{eq: dist_sub}
    x \lor (y \land z) = (x \lor y) \land (x \lor z) \quad \text{and} \quad
    x \land (y \lor z) = (x \land y) \lor (x \land z).
\end{align}
Let $X$, $Y$, and $Z$ be unspecified terms---either single variables, or more complicated compound terms---and consider a particular term $\hat{p}(X,Y,Z)$. Then we say that the distributive law \emph{applies} to $\hat{p}(X,Y,Z)$ if $\hat{p}(X,Y,Z)$ is of one of the following forms:
\begin{alignat}{2} \label{eq: dist_check}
    &X \lor (Y \land Z) \quad \text{or} \quad &&(X \lor Y) \land (X \lor Z) \quad \text{or} \\
    &X \land (Y \lor Z) \quad \text{or} \quad &&(X \land Y) \lor (X \land Z). \nonumber
\end{alignat}
If this is the case, we say that the distributive law is then \emph{applied} to $\hat{p}(X,Y,Z)$ by applying the appropriate symbolic substitution to $\hat{p}(X,Y,Z)$ from \eqref{eq: dist_sub}. If multiple forms from \eqref{eq: dist_check} are appropriate, one is chosen at random.

\cref{alg: equiv_expr} essentially randomly applies a certain number of these transformations to random subterms of a particular starting term, yielding a chain of equivalent terms. We provide an example of such a chain here, starting from the operational polynomial $p(x,y,z) = x \land (y \lor z)$ and proceeding for three steps:
\begin{align*}
    x \land (y \lor z) &= x \land (z \lor y) \tag{commutativity}\\
    &= (x \land z) \lor (x \land y) \tag{distributivity}\\
    &= ((x \land (x \lor y)) \land z) \lor (x \land y) \tag{absorption} \\
    &\eqqcolon q(x,y,z).
\end{align*}

In our self-consistency experiments, we then compare the realizations of our learned operations on the final term $q(x,y,z)$ against those of the starting term $p(x,y,z)$.

\begin{center}
\vspace*{1cm}    
\begin{minipage}{.9\linewidth}
\begin{algorithm}[H]
\caption{Random term generation}
\label{alg: rand_expr}
\DontPrintSemicolon
\KwIn{Number of symbols $\ell$}
\KwOut{Term $p(x_1, \dots, x_{\ell})$}

$\hat{P} \leftarrow (x_1, \dots, x_{\ell})$ \Comment*[r]{Initialize list of terms}

\While{$\textup{length}(\hat P) > 1$}{
    $i \leftarrow$ random integer in $\{1, \dots, \text{length}(\hat P)\}$
    
    $\hat p_1 \leftarrow \text{pop}(\hat P, i)$ \Comment*[r]{Extract $i$th term and remove from list}
    
    $j \leftarrow$ random integer in $\{1, \dots, \text{length}(\hat P)\}$
    
    $\hat p_2 \leftarrow \text{pop}(\hat P, j)$ \Comment*[r]{Extract $j$th term and remove from list}
    
    sym $\leftarrow$ random operation symbol in $\{\cap, \cup\}$ \Comment*[r]{Select operation symbol}
    
    $\hat q \leftarrow \hat p_1\ \text{sym}\ \hat p_2$ \Comment*[r]{Apply operation symbol}
    
    append($\hat P$, $\hat q$) \Comment*[r]{Append new term to list}
}
\Return $\hat P[1]$ \Comment*[r]{Return remaining term in list}
\end{algorithm}
\end{minipage}
\end{center}
\vspace*{1.3cm}

\begin{center}
\begin{minipage}{.9\linewidth}
\SetKwComment{Comment}{\# }{}
\begin{algorithm}[H]
\caption{Equivalent term generation}
\label{alg: equiv_expr}
\DontPrintSemicolon
\KwIn{Starting term $p(x_1, \dots, x_{\ell})$, number of law applications $J$}
\KwOut{Equivalent term $q(x_1, \dots, x_{\ell})$}

\For{$\_ \leftarrow 1$ \KwTo $J$}{
    Recursively extract all subterms of $p$ into a collection (tree search): \\
    $\hat P \leftarrow (\hat p_1, \dots, \hat p_n)$ \Comment*[r]{Specific subterms in $p$}
    $\;\; \hat I \leftarrow (\hat i_1, \dots, \hat i_n)$ \Comment*[r]{Indices for each $\hat p$ in $p$}

    $\textup{laws} \leftarrow (\textup{``associativity''}, \textup{``commutativity''}, \textup{``absorption'', \textup{``distributivity''}})$
    
    Randomly shuffle $\hat P$ and $\hat I$ (with the same shuffle)

    Randomly shuffle $\textup{laws}$
    
    \For{$j \leftarrow 1$ \KwTo $\textup{length}(\textup{laws})$}{
        law $\leftarrow$ laws$[j]$

        \For{$k \leftarrow 1$ \KwTo $\textup{length}(\hat P)$}{
            $\hat p, \hat i$ $\leftarrow$ $\hat P[k]$, $\hat I[k]$
            
            \If{\textup{law} applies to $\hat p$}{
                $\hat q \leftarrow \textup{law}$ applied to $\hat p$\;

                $p \leftarrow \hat q$ substituted for $\hat p$ at $\hat i$ in $\hat P$
            }
        }
    }
}
\Return $\hat P$
\end{algorithm}
\end{minipage}
\end{center}

\clearpage
\subsection{Additional visualizations}  \label{app: experiment_plots} 
This section provides additional tables and graphics for our experiments.

\paragraph{Table of operations.}
\cref{tab: algebra_long} lists all combinations of candidate operations that we trained. For each combination, we numerically tested whether the $8$ distributive lattice laws in \cref{tab: lattice_laws} are satisfied. Summing the number of satisfied laws for a particular pair of combination yields the count (\#) column.

The symmetric directly parameterized baseline satisfies two laws (commutativity in both directions), while the naive MLP satisfies no laws. These are not listed in the table.

\begin{table*}[ht]
\centering
\caption{\label{tab: algebra_long} Candidate operations on the mirrored space with the satisfied distributive lattice laws. Due to distributive lattice symmetries, we have two laws for each column (e.g., $a \cap^{\M} b = b \cap^{\M} a$ and $a \cup^{\M} b = b \cup^{\M} a$). The first row imposes a Riesz algebra structure on the mirrored space. The second column counts how many laws are satisfied by a particular pair of candidate operations.
\vspace*{0.2cm}
}

\begin{tabular}{llrcccc}
\toprule
 Operations                        &                                   &   \# & Commutativity ($^*$)   & Associativity ($^*$)   & Absorption ($^*$)   & Distributivity ($^*$)   \\
\midrule
 $\cap^{\M} = \max$                & $\cup^{\M} = \min$                &    8 & \cmark \cmark          & \cmark \cmark          & \cmark \cmark       & \cmark \cmark           \\
 $\cap^{\M} = \max$                & $\cup^{\M} = \odot$               &    6 & \cmark \cmark          & \cmark \cmark          & \xmark \cmark       & \cmark \xmark           \\
 $\cap^{\M} = \min$                & $\cup^{\M} = +$                   &    6 & \cmark \cmark          & \cmark \cmark          & \xmark \cmark       & \cmark \xmark           \\
 $\cap^{\M} = \max$                & $\cup^{\M} = +$                   &    5 & \cmark \cmark          & \cmark \cmark          & \xmark \xmark       & \cmark \xmark           \\
 $\cap^{\M} = \min$                & $\cup^{\M} = \odot$               &    5 & \cmark \cmark          & \cmark \cmark          & \xmark \xmark       & \cmark \xmark           \\
 $\cap^{\M} = \min$                & $\cup^{\M} = +_s$                 &    5 & \cmark \cmark          & \cmark \xmark          & \xmark \cmark       & \cmark \xmark           \\
 $\cap^{\M} = +$                   & $\cup^{\M} = \odot$               &    5 & \cmark \cmark          & \cmark \cmark          & \xmark \xmark       & \cmark \xmark           \\
 $\cap^{\M} = \max$                & $\cup^{\M} = +_s$                 &    4 & \cmark \cmark          & \cmark \xmark          & \xmark \xmark       & \cmark \xmark           \\
 $\cap^{\M} = \min$                & $\cup^{\M} = \times_{\text{mat}}$ &    4 & \cmark \xmark          & \cmark \cmark          & \xmark \cmark       & \xmark \xmark           \\
 $\cap^{\M} = +$                   & $\cup^{\M} = \times_{\text{mat}}$ &    4 & \cmark \xmark          & \cmark \cmark          & \xmark \xmark       & \cmark \xmark           \\
 $\cap^{\M} = \odot$               & $\cup^{\M} = +_s$                 &    4 & \cmark \cmark          & \cmark \xmark          & \xmark \xmark       & \xmark \cmark           \\
 $\cap^{\M} = \max$                & $\cup^{\M} = -$                   &    3 & \cmark \xmark          & \cmark \xmark          & \xmark \cmark       & \xmark \xmark           \\
 $\cap^{\M} = \max$                & $\cup^{\M} = \times_{\text{mat}}$ &    3 & \cmark \xmark          & \cmark \cmark          & \xmark \xmark       & \xmark \xmark           \\
 $\cap^{\M} = \max$                & $\cup^{\M} = +_c$                 &    3 & \cmark \xmark          & \cmark \xmark          & \xmark \xmark       & \cmark \xmark           \\
 $\cap^{\M} = \min$                & $\cup^{\M} = +_c$                 &    3 & \cmark \xmark          & \cmark \xmark          & \xmark \xmark       & \cmark \xmark           \\
 $\cap^{\M} = +$                   & $\cup^{\M} = +_s$                 &    3 & \cmark \cmark          & \cmark \xmark          & \xmark \xmark       & \xmark \xmark           \\
 $\cap^{\M} = \odot$               & $\cup^{\M} = \times_{\text{mat}}$ &    3 & \cmark \xmark          & \cmark \cmark          & \xmark \xmark       & \xmark \xmark           \\
 $\cap^{\M} = +_s$                 & $\cup^{\M} = \times_{\text{mat}}$ &    3 & \cmark \xmark          & \xmark \cmark          & \xmark \xmark       & \cmark \xmark           \\
 $\cap^{\M} = \min$                & $\cup^{\M} = -$                   &    2 & \cmark \xmark          & \cmark \xmark          & \xmark \xmark       & \xmark \xmark           \\
 $\cap^{\M} = +$                   & $\cup^{\M} = -$                   &    2 & \cmark \xmark          & \cmark \xmark          & \xmark \xmark       & \xmark \xmark           \\
 $\cap^{\M} = +$                   & $\cup^{\M} = +_c$                 &    2 & \cmark \xmark          & \cmark \xmark          & \xmark \xmark       & \xmark \xmark           \\
 $\cap^{\M} = -$                   & $\cup^{\M} = \odot$               &    2 & \xmark \cmark          & \xmark \cmark          & \xmark \xmark       & \xmark \xmark           \\
 $\cap^{\M} = \odot$               & $\cup^{\M} = +_c$                 &    2 & \cmark \xmark          & \cmark \xmark          & \xmark \xmark       & \xmark \xmark           \\
 $\cap^{\M} = -$                   & $\cup^{\M} = +_s$                 &    1 & \xmark \cmark          & \xmark \xmark          & \xmark \xmark       & \xmark \xmark           \\
 $\cap^{\M} = -$                   & $\cup^{\M} = \times_{\text{mat}}$ &    1 & \xmark \xmark          & \xmark \cmark          & \xmark \xmark       & \xmark \xmark           \\
 $\cap^{\M} = +_s$                 & $\cup^{\M} = +_c$                 &    1 & \cmark \xmark          & \xmark \xmark          & \xmark \xmark       & \xmark \xmark           \\
 $\cap^{\M} = \times_{\text{mat}}$ & $\cup^{\M} = +_c$                 &    1 & \xmark \xmark          & \cmark \xmark          & \xmark \xmark       & \xmark \xmark           \\
 $\cap^{\M} = -$                   & $\cup^{\M} = +_c$                 &    0 & \xmark \xmark          & \xmark \xmark          & \xmark \xmark       & \xmark \xmark           \\
\bottomrule
\end{tabular}

\end{table*}

\newpage
\paragraph{Operation performance and self-consistency scatter plots.}
We supplement our plots in \cref{fig: line} with two in-depth scatter plots which elucidate interesting details in the data.

\cref{fig: scatter_a} presents the same \cref{sec: sat_laws} experimental data as in \cref{fig: line_a}. Namely, \cref{fig: scatter_a} teases apart the influence of the length $\ell$ of the random symbol term, which is simply averaged in \cref{fig: line_a}. The IoU performance metric, originally on the $y$-axis of \cref{fig: line_a}, is now visualized in the color bar. The $y$-axis of \cref{fig: scatter_a} starts with two-symbol terms (either $x \cup y$ or $y \cup x$) up to terms with $10$ symbols. Each ``column'' in \cref{fig: scatter_a} corresponds to one specific choice of algebra operations, which is represented as just a dot in \cref{fig: line_a}. We annotate a few explicitly in both plots. These algebras are then grouped by how many distributive lattice laws they satisfy, which can be found in \cref{tab: algebra_long}.

\cref{fig: scatter_a} shows a moderate performance degradation as the length of the random term increases. This is less pronounced for methods with a low IoU to begin with, but can be more easily seen with better-performing methods, especially in the algebras with $5$ and $6$ satisfied laws. These algebras perform similarly to the Riesz algebra for terms with only a few symbols; however, as the length of the term increases the Riesz algebra emerges as a favorite. We believe this is attributable to the fact that the Riesz algebra exclusively satisfies all the desired distributive lattice laws.

\cref{fig: scatter_b} similarly corresponds to \cref{fig: line_b} and the experiment in \cref{sec: equivalent_expr}. Each column in \cref{fig: scatter_b} is a particular algebra, corresponding to one line in the line plot of \cref{fig: line_b}. The color bar encodes the same self-consistency metric discussed in \cref{sec: equivalent_expr} and plotted on the $y$-axis of \cref{fig: line_b}. However, the $y$-axis of \cref{fig: scatter_b} corresponds to the $x$-axis of \cref{fig: line_b}, and the $x$-axis of \cref{fig: scatter_b} simply serves to arrange the different algebras in order with their alignment to the source algebra.

\cref{fig: scatter_b} shows that all methods besides the Riesz algebra struggle to maintain self-consistency for equivalent terms. Naturally, for unchanged terms ($0$ on the $y$-axis) all methods are deterministic and hence self-consistent. The Riesz algebra is also perfectly self-consistent for all equivalent terms, as a direct consequence of \cref{thm: source_latent}. All other algebras degrade significantly as the equivalent terms become more complex. This stems from the fact that, unlike the Riesz algebra, all other operations must implicitly learn the underlying source algebra laws. This is always an inexact process and produces compounding errors for more complex terms.

\begin{figure*}[ht]
\centering
\begin{subfigure}[b]{0.48\linewidth}
    \centering
    \includegraphics[width=\linewidth]{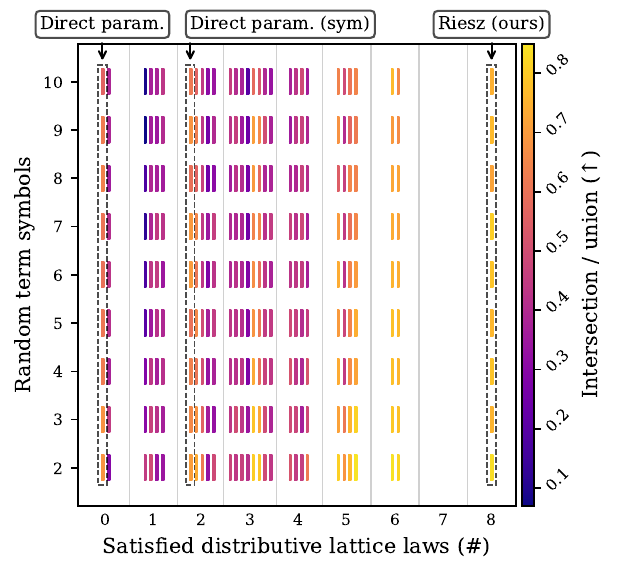}
    \caption{\label{fig: scatter_a}}
\end{subfigure}
\quad
\begin{subfigure}[b]{0.48\linewidth}
    \centering
    \includegraphics[width=\linewidth]{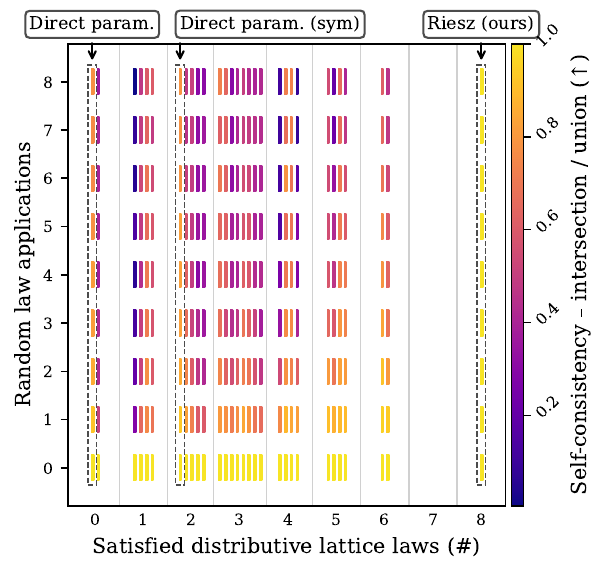}
    \caption{\label{fig: scatter_b}}
\end{subfigure}

\caption{\label{fig: scatter} \subref{fig: scatter_a}) Average learned operation performance, plotted on the color bar, against number of satisfied distributive lattice laws and random term length. \subref{fig: scatter_b}) Average self-consistency, plotted on the color bar, against number of satisfied distributive lattice laws and number of random law applications.}
\end{figure*}

\newpage
\paragraph{Breakdown of specific laws.}
This series of plots breaks down the \cref{sec: sat_laws} data, plotted in \cref{fig: line_a} and \cref{fig: scatter_a}, at a more granular level. Instead of only considering the aggregate number of laws satisfied, \cref{fig: group} analyzes the satisfaction of specific laws and its impact on learned operation performance. 

Each point in any of the \cref{fig: group} subplots represents a single algebra, with its color encoding the mean operation performance averaged over different term lengths, as in the $y$-axis of \cref{fig: line_a}. For two particular distributive lattice laws, we group the algebras by the degree to which they satisfy each of the laws: this can include both of the forms in \cref{tab: lattice_laws} (\cmark \cmark), just one or the other (\cmark \xmark / \xmark \cmark), or neither (\xmark \xmark). As with the other plots, we annotate our best-performing structural transport net and the two direct latent parameterization baselines.

\cref{fig: group} broadly shows that satisfaction of any laws is positively related to learned operation performance. Specifically, moving from left-to-right and from bottom-to-top generally results in an increased average IoU score. However, conditioned on a particular law being satisfied, certain other laws seem to be less important. We enumerate a few such observations here:
\begin{enumerate}
    \item Conditioned on the level of distributivity, increasing absorption seems to yield minimal benefits (\cref{fig: group_a}).
    \item Conditioned on the level of commutativity, increasing absorption seems to yield minimal benefits (\cref{fig: group_d}).
    \item Completely satisfying both commutativity and associativity is highly indicative of strong performance, but only partial satisfaction of either results in substantial degradation (\cref{fig: group_e}).
\end{enumerate}

Certain combinations of law satisfaction were not covered by our set of candidate operations. For example, no operations in $\C \times \C$ resulted in an algebra which was completely distributive and not at all absorptive \cref{fig: group_a}. Future work can consider more sophisticated methods for constructing a set of candidate operations $\C$ which results in complete coverage of all possible law satisfactions.

\begin{figure*}[ht]
\centering
\newcommand{\figurewidth}{0.45\linewidth}
\newcommand{\compressspace}{\vspace{-0.1cm}}
\begin{subfigure}[b]{\figurewidth}
    \centering
    \includegraphics[width=\linewidth]{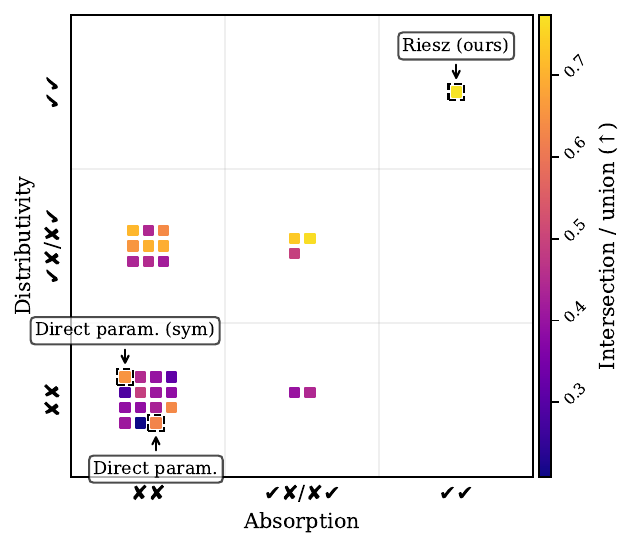}
    \caption{\label{fig: group_a}}
\end{subfigure}
\quad
\begin{subfigure}[b]{\figurewidth}
    \centering
    \includegraphics[width=\linewidth]{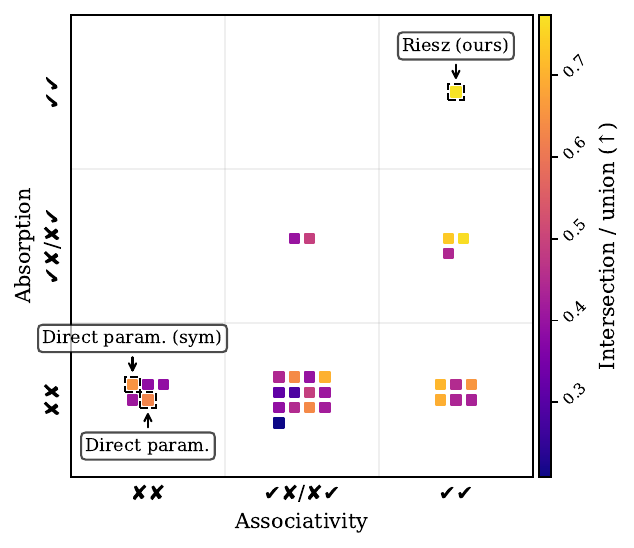}
    \caption{\label{fig: group_b}}
\end{subfigure} \\
\begin{subfigure}[b]{\figurewidth}
    \centering
    \includegraphics[width=\linewidth]{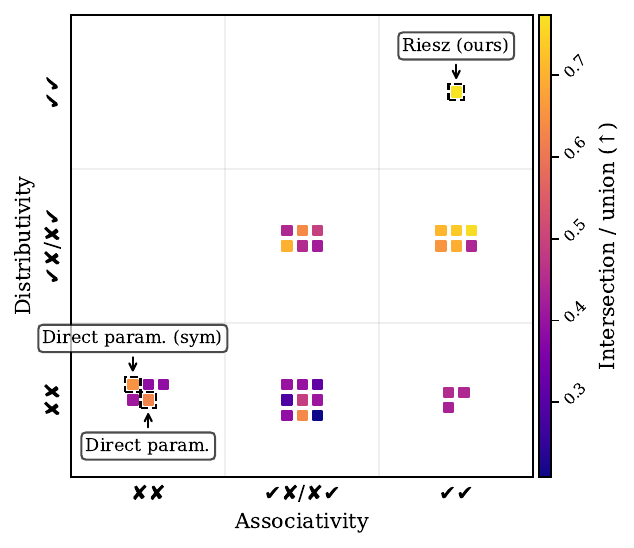}
    \caption{\label{fig: group_c}}
\end{subfigure}
\quad
\begin{subfigure}[b]{\figurewidth}
    \centering
    \includegraphics[width=\linewidth]{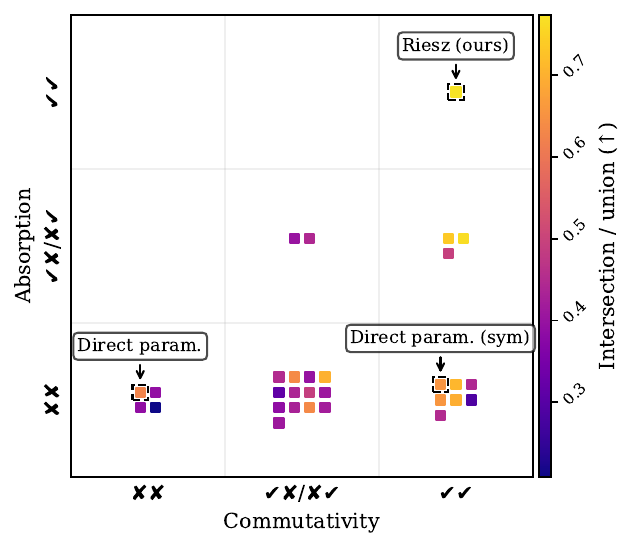}
    \caption{\label{fig: group_d}}
\end{subfigure} \\
\begin{subfigure}[b]{\figurewidth}
    \centering
    \includegraphics[width=\linewidth]{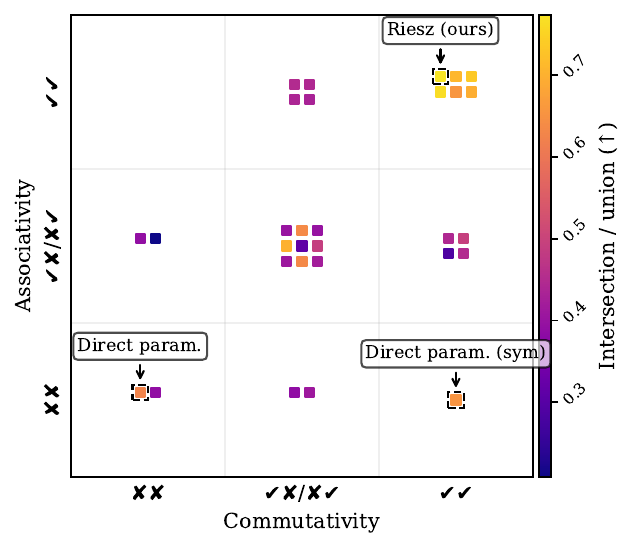}
    \caption{\label{fig: group_e}}
\end{subfigure}
\quad
\begin{subfigure}[b]{\figurewidth}
    \centering
    \includegraphics[width=\linewidth]{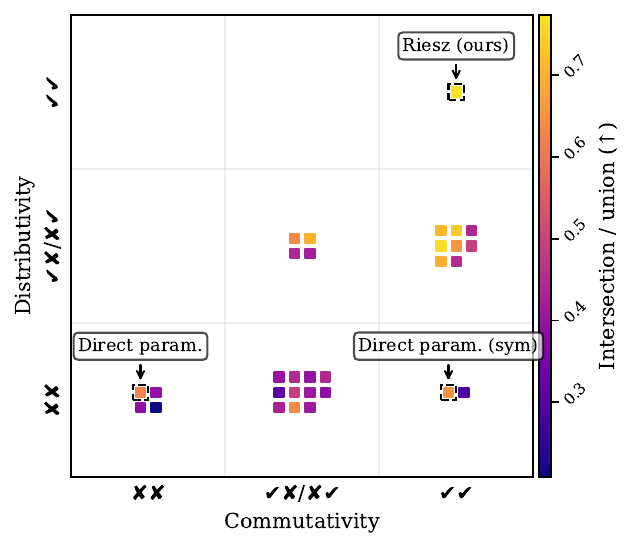}
    \caption{\label{fig: group_f}}
\end{subfigure} \\

\caption{\label{fig: group} The performance of different algebras, grouped by their satisfaction of specific law combinations.}
\end{figure*}

\end{document}